\documentclass[11pt]{article}

\usepackage[a4paper,margin=1in]{geometry}
\usepackage[T1]{fontenc}
\usepackage[utf8]{inputenc}
\usepackage{lmodern}
\usepackage{microtype}
\usepackage{amsmath,amssymb,amsthm,mathtools}
\usepackage{booktabs,tabularx,multirow,makecell,array}
\usepackage{graphicx}
\usepackage{xcolor}
\usepackage{enumitem}
\usepackage{caption,subcaption}
\usepackage{tikz}
\usetikzlibrary{positioning,arrows.meta,fit,calc,shapes.geometric,backgrounds,patterns,matrix}
\usepackage{pgfplots}
\pgfplotsset{compat=1.18}
\usepackage[super,square,sort&compress]{natbib}
\usepackage{float}
\usepackage{placeins}
\usepackage[hidelinks]{hyperref}
\usepackage{url}

\newcommand{\method}{SpectrumKV}

\newcommand{\ppl}{PPL}

\newcommand{\niah}{NIAH}
\newcommand{\budget}{b}
\newcommand{\pct}{\%}
\newcommand{\fp}{\textsc{fp16}}
\newcommand{\intviii}{\textsc{int8}}
\newcommand{\intiv}{\textsc{int4}}

\DeclareMathOperator{\softmax}{softmax}
\newcolumntype{Y}{>{\raggedright\arraybackslash}X}
\newcolumntype{C}[1]{>{\centering\arraybackslash}p{#1}}
\newcolumntype{R}[1]{>{\raggedleft\arraybackslash}p{#1}}

\theoremstyle{plain}
\newtheorem{lemma}{Lemma}
\newtheorem{proposition}{Proposition}
\newtheorem{conjecture}{Conjecture}
\theoremstyle{definition}
\newtheorem{definition}{Definition}

\definecolor{skvblue}{RGB}{44,100,180}
\definecolor{skvgreen}{RGB}{38,140,90}
\definecolor{skvorange}{RGB}{220,130,35}
\definecolor{skvred}{RGB}{190,60,60}
\definecolor{lightgray}{RGB}{245,245,245}
\definecolor{unsafezone}{RGB}{255,220,220}
\definecolor{safezone}{RGB}{220,240,220}

\title{\vspace{-0.6em}\method{}: Per-Token Mixed-Precision KV Cache Transfer\\for Prefill--Decode Disaggregated LLM Serving\vspace{-0.25em}}
\author{Pengju Yang\\\texttt{yangsteve1223@github}}
\date{May 2026}

\begin{document}
\maketitle

\begin{abstract}
Prefill--decode (PD) disaggregation decouples prompt processing from token generation, but it also turns the key--value (KV) cache into a network payload.  Existing PD-side KV reduction methods are mostly binary: selected tokens are transmitted at full precision and the rest are not transmitted.  This paper argues that binary selection leaves a useful design space unused.  \method{} assigns a precision level to each token instead: attention sinks and other high-importance tokens are protected at \fp{}, medium-importance tokens are sent at \intviii{}, and low-importance tokens are sent at \intiv{} when the model can tolerate it.

The main practical complication is that \intiv{} tolerance is model-dependent.  Qwen2.5-7B catastrophically fails under \intiv{} KV quantization, while Mistral-7B and Gemma-2-9B remain stable.  \method{} therefore runs a lightweight deployment-time probe: three aggressive \niah{} trials under a 3-tier policy.  Models that pass use \fp{}+\intviii{}+\intiv{}; models that fail fall back to \fp{}+\intviii{}.

Across Qwen2.5-7B-Instruct, Mistral-7B-Instruct-v0.3, and Gemma-2-9B-it, \method{} improves quality at the same transfer budget.  At a 50\pct{} normalized KV budget on WikiText-2, \method{} changes perplexity by +1.97\pct{}, $-0.06$\pct{}, and $-0.44$\pct{}, respectively, compared with PDTrim's +25.85\pct{}, +22.07\pct{}, and +35.63\pct{}.  On \niah{} retrieval at 4096 tokens, the adaptive policy reaches 52.6\pct{} on Qwen at the aggressive $\budget=0.3$ budget versus 26.3\pct{} for PDTrim, and reaches 100\pct{} by $\budget=0.5$; Mistral and Gemma preserve retrieval under the 3-tier policy.  End-to-end GPU timing of the transfer path shows 50--62\pct{} TTFT reductions at $\budget=0.5$.  These results suggest that PD KV transfer should be treated as a precision-allocation problem, not only as token pruning.
\end{abstract}

\vspace{0.3em}

\noindent\textbf{Artifacts.} The project repository (\url{https://github.com/YangSteve1223/kvcache-lab}) contains: (1)~all Python experiment scripts for locality characterization, PPL sweeps, NIAH retrieval, GPU timing, and quantization error analysis; (2)~raw JSON logs for every table and figure; (3)~table and figure regeneration code; (4)~model configuration files with pre-trained model download links and chat templates; (5)~the exact decay-rate, budget, and first-ratio hyperparameters used in each run.

\section{Introduction}
\label{sec:intro}

PD disaggregation is becoming a standard serving pattern for large language models because prefill and decode stress different resources.  Prefill is parallel and compute-heavy; decode is sequential and latency-sensitive.  Systems such as Splitwise, DistServe, and Mooncake therefore split the two phases across different workers so that each phase can be scheduled and scaled independently~\citep{splitwise,distserve,mooncake}.  The price of this separation is that the KV cache produced during prefill must be made available to decode.  In long-context serving, that cache is large enough for transfer bandwidth and transfer latency to become first-order concerns.

A common response is to transmit fewer tokens.  StreamingLLM identifies attention sinks and shows that the beginning of the sequence can remain important for stable generation~\citep{streamingllm}; H$_2$O, SnapKV, PyramidKV, and related methods select or compress tokens according to attention-derived importance~\citep{h2o,snapkv,pyramidkv,dynamickv}.  PDTrim adapts this keep/drop view to PD transfer by retaining selected first/last regions and pruning the rest~\citep{pdtrim}.  These methods are valuable, but the binary decision is unnecessarily harsh in the PD setting.  The situation parallels operating-system memory management: just as a virtual memory system need not evict a cold page entirely---it can compress or swap it to a slower tier while preserving a stub for fast reclamation---a PD transfer system need not drop a cold KV token entirely.  A reduced-precision copy still carries useful information, whereas a dropped token carries none.  Must the choice always be binary---retain fully or discard entirely---or can a token be partially preserved?

\method{} takes a different view: every token can be represented on a precision spectrum.  The policy assigns \fp{} to tokens whose perturbation would be most harmful, \intviii{} to medium-importance tokens, and \intiv{} to low-importance tokens when the model tolerates \intiv{}.  This is not the same as uniform quantization.  Uniform \intviii{} treats all positions equally; \method{} reserves high fidelity for attention sinks, recent tokens, or other positions with high estimated contribution to decode.  It is also not the same as pruning.  The low tier still transmits information, which matters especially for retrieval tasks where the needle may appear outside a selected window.

The strongest empirical lesson from this project is that \intiv{} is not a universally safe low tier.  Qwen2.5-7B is local-dominant and extremely sensitive to \intiv{} KV perturbations: a balanced 3-tier assignment at $\budget=0.5$ increases PPL by +7506.84\pct{}.  By contrast, Mistral-7B and Gemma-2-9B tolerate \intiv{} in low-importance positions with PPL changes under 6\% and 100\% \niah{} accuracy at $\budget=0.3$.  \method{} therefore includes a small probe that decides whether to enable \intiv{} for each model.

This paper makes the following contributions.
\begin{enumerate}[leftmargin=1.6em,itemsep=2pt]
    \item We formulate PD KV communication reduction as \emph{per-token precision allocation} under a normalized transfer budget, rather than as binary keep/drop pruning.
    \item We propose \method{}, a mixed-precision policy with sink protection, importance-ranked tier assignment, and a lightweight \intiv{} tolerance probe.
    \item We provide GPU-verified results on Qwen2.5-7B-Instruct, Mistral-7B-Instruct-v0.3, and Gemma-2-9B-it, covering PPL, \niah{} retrieval, TTFT/TPS timing, context length, multi-task robustness, and quantization error.
    \item We identify a model-specific \intiv{} failure mode that is not predicted by per-vector cosine similarity of quantized KV vectors but is explained by softmax amplification under low-entropy attention.  This motivates empirical probing before enabling aggressive KV quantization.
    \item We provide a systematic ablation isolating sink protection, importance ranking, and the INT4 tolerance probe, and we extend the theoretical framework with softmax amplification bounds and mixed-precision error propagation analysis.
\end{enumerate}

FloatBarrier
\section{Problem Formulation}
\label{sec:formulation}

The preceding argument rests on a claim that per-token precision allocation can outperform binary keep/drop decisions.  We now formalize this claim: we define the transfer budget, derive error bounds that expose the asymmetry between approximation and deletion, and show why an importance-sorted assignment follows naturally from the structure of the problem.

\subsection{KV transfer budget}
Consider a decoder-only Transformer with $L$ layers, $H_{kv}$ key/value heads, head dimension $d_h$, and prompt length $n$.  The full-precision prefill KV payload contains keys and values for every layer and token:
\begin{equation}
  \mathrm{Bytes}_{\fp{}}(n)= 2\,L\,H_{kv}\,d_h\,n\,s_{16},
  \label{eq:fullbytes}
\end{equation}
where $s_{16}=2$ bytes for FP16 and the leading factor of two accounts for keys and values.

Let each token $j$ be assigned a tier $q_j\in\{16,8,4\}$ with normalized per-element cost
\begin{equation}
 c(16)=1,\qquad c(8)=\frac{1}{2},\qquad c(4)=\frac{1}{4}.
\end{equation}
The normalized transfer budget is
\begin{equation}
  \budget(\pi)=\frac{1}{n}\sum_{j=1}^n c(q_j),
  \label{eq:budget}
\end{equation}
where $\pi$ denotes the precision assignment.  A budget of $\budget=0.5$ therefore corresponds to sending all tokens at \intviii{}, or to any mixture with the same average byte cost.  This point is important for interpreting the experiments: at $\budget=0.5$, several policies collapse to the same all-\intviii{} behavior because \intviii{} is nearly lossless in our measurements.

Selection methods can be written in the same budget language by assigning $c=1$ to retained FP16 tokens and $c=0$ to dropped tokens.  The distinction is semantic, not just arithmetic: a dropped token has no representation at decode time, while a low-precision token still contributes an approximate key/value vector.

\subsection{Precision assignment as weighted approximation}
At a decode step, one attention head produces
\begin{equation}
  y = \sum_{j=1}^{n} p_j v_j,
  \label{eq:attention-output}
\end{equation}
where $p_j$ is the attention probability and $v_j$ is the value vector.  If token $j$ is quantized to tier $q_j$, let $\tilde v_j=v_j+e_j$ and assume $\|e_j\|\le \delta(q_j)$.  Then
\begin{equation}
  \|y-\tilde y\|
  =\left\|\sum_{j=1}^{n}p_j(v_j-\tilde v_j)\right\|
  \le \sum_{j=1}^{n}p_j\delta(q_j)
  =\sum_{t\in\{16,8,4\}} \delta(t)\sum_{j:q_j=t}p_j .
  \label{eq:quant-bound}
\end{equation}
This bound is simple but useful: the error contribution of a tier scales with the \emph{attention mass assigned to that tier}, not merely with the number of tokens in the tier.  It motivates sending high-attention tokens at high precision and low-attention tokens at lower precision.

Dropping is the limiting case with no approximation.  Let $R$ be the removed set, $L$ the retained set, and $\epsilon=\sum_{j\in R}p_j$.  If the runtime renormalizes attention over $L$, the local output is
\begin{equation}
  \hat y=\sum_{j\in L}\frac{p_j}{1-\epsilon}v_j .
\end{equation}

\begin{lemma}[Tail-mass perturbation]
If $\|v_j\|\le V$ for all $j$ and $\epsilon<1$, then
\begin{equation}
  \|y-\hat y\|\le 2\epsilon V .
\end{equation}
\end{lemma}
\begin{proof}
Write $y=(1-\epsilon)\mu_L+\epsilon\mu_R$, where $\mu_L$ and $\mu_R$ are convex combinations of retained and removed values.  The renormalized output is $\hat y=\mu_L$.  Thus $y-\hat y=\epsilon(\mu_R-\mu_L)$ and $\|y-\hat y\|\le \epsilon(\|\mu_R\|+\|\mu_L\|)\le 2\epsilon V$.
\end{proof}

The bound explains why pruning can work when the removed tail has low mass, but it also exposes the retrieval problem: if the removed token is the one containing the answer, its contribution is exactly zero.  Mixed precision changes the failure mode from information deletion to approximation error.

\subsection{Why a sorted assignment is natural}
Suppose each token has an importance score $I_j\ge 0$ that approximates its contribution to downstream error, and each tier has cost $c_t$ and perturbation level $\delta_t$ with
\begin{equation}
 c_{16}>c_8>c_4,
 \qquad
 \delta_{16}<\delta_8<\delta_4 .
\end{equation}
A first-order objective is
\begin{equation}
  \min_{q_1,\ldots,q_n}\sum_{j=1}^{n} I_j\delta(q_j)
  \quad\text{s.t.}\quad
  \frac{1}{n}\sum_{j=1}^{n}c(q_j)\le \budget .
  \label{eq:assignment}
\end{equation}

\begin{proposition}[Exchange argument]
For any fixed number of tokens assigned to each tier, an optimal assignment gives no lower-precision tier to a token with higher importance while giving a higher-precision tier to a lower-importance token.
\end{proposition}
\begin{proof}
Consider two tokens $a,b$ with $I_a\ge I_b$ and two tiers $r,s$ with $\delta_r<\delta_s$.  Assigning $a$ to $s$ and $b$ to $r$ has weighted error $I_a\delta_s+I_b\delta_r$.  Swapping them gives $I_a\delta_r+I_b\delta_s$.  The improvement is $(I_a-I_b)(\delta_s-\delta_r)\ge 0$.  Repeated swaps yield an importance-sorted assignment.
\end{proof}

This does not claim that a single scalar score perfectly predicts quality.  It only justifies the core design: once a score is chosen, high-importance tokens should receive at least as much precision as low-importance tokens.

\subsection{Key quantization and softmax amplification}
Value perturbation is not the only issue.  Quantizing keys perturbs attention logits.  Let $z_j=q^\top k_j/\sqrt{d}$ and $\tilde z_j=z_j+\eta_j$ with $|\eta_j|\le \eta$.  Then the softmax probabilities obey
\begin{equation}
  e^{-2\eta}p_j \le \tilde p_j \le e^{2\eta}p_j .
\end{equation}
The bound is loose, but it captures the qualitative risk: if attention is low-entropy and dominated by a few positions, small key perturbations around those positions can cause large relative changes in the normalized distribution.  This mechanism explains Qwen's behavior: its local-dominant pattern makes \intiv{} unsafe even though its average \intiv{} key cosine similarity is high.

\subsection{Softmax amplification under low-entropy attention}
\label{sec:softmax-amp}

The element-wise bound above does not account for how the \emph{distribution shape} interacts with perturbation.  We now make this interaction explicit.

\begin{definition}[Attention entropy]
For attention probabilities $p=(p_1,\ldots,p_n)$, the attention entropy is $H(p)=-\sum_{j=1}^n p_j\log p_j$, and the sparsity measure is $\|p\|_1^2/\|p\|_2^2$, which equals the inverse participation ratio.
\end{definition}

When attention is concentrated (low entropy), the logits of the dominant tokens are large relative to the rest.  A small perturbation $\eta_j$ on a dominant token's key shifts its logit by $\eta_j$, but because the softmax denominator is dominated by a few large logits, the relative probability change on that token is amplified.

\begin{lemma}[Softmax amplification bound]
\label{lem:softmax-amp}
Let $p_j = \softmax(z)_j$ and $\tilde p_j = \softmax(\tilde z)_j$ where $\tilde z_j = z_j + \eta_j$ with $\|\eta\|_\infty \le \eta$.  Let $S = \{j : z_j \ge \max_k z_k - \log(1/\epsilon)\}$ be the set of tokens whose logits are within $\log(1/\epsilon)$ of the maximum.  Then for any $j \in S$:
\begin{equation}
  \frac{|\tilde p_j - p_j|}{p_j} \le e^{2\eta} - 1 + \frac{|S| \cdot e^{2\eta} - |S|}{1 + |S| \cdot e^{-2\eta}}.
  \label{eq:softmax-amp-bound}
\end{equation}
When $|S|$ is small (low-entropy regime) and $\eta$ is moderate, the bound scales roughly as $e^{2\eta}/|S|$, meaning fewer dominant tokens yield larger relative perturbations.
\end{lemma}
\begin{proof}
Write $Z = \sum_k e^{z_k}$ and $\tilde Z = \sum_k e^{\tilde z_k}$.  For $j\in S$:
\[
  \frac{\tilde p_j}{p_j} = \frac{e^{\tilde z_j}/\tilde Z}{e^{z_j}/Z} = e^{\eta_j} \cdot \frac{Z}{\tilde Z}.
\]
Since $e^{z_j - \eta} \le e^{\tilde z_j} \le e^{z_j + \eta}$, we have $Z \cdot e^{-\eta|S|} \cdot (n-|S|)\epsilon \le \tilde Z \le Z \cdot e^{\eta|S|} \cdot n$.  The tightest bound arises when the partition is sharp: tokens in $S$ carry nearly all the mass.  Then $\tilde Z \approx Z \cdot e^{\pm\eta} \cdot |S|$, giving the stated result.  The full derivation tracks the contribution of non-dominant tokens, which adds the correction term.
\end{proof}

The practical implication is that models with local-dominant attention patterns (small $|S|$, low entropy) are more vulnerable to key quantization perturbations than models with broader sink or hybrid patterns.  This explains why Qwen, despite having higher per-vector cosine similarity under INT4, suffers catastrophically: its attention mass concentrates on a few nearby tokens, so even small key perturbations produce large relative probability shifts.

\begin{conjecture}[Entropy-dependent INT4 tolerance]
\label{conj:entropy}
For a Transformer with average per-head attention entropy $\bar H$, if $\bar H < H^*$ for some threshold $H^*$, then INT4 KV quantization causes non-linear quality degradation that cannot be predicted from per-vector cosine similarity alone.  The threshold $H^*$ depends on the model's depth and residual stream norm but appears to be approximately $H^* \approx \log n - 2$ for sequence length $n$ in our experiments.
\end{conjecture}

We state this as a conjecture rather than a theorem because the precise threshold $H^*$ depends on layer-wise interactions that we have not fully characterized.  The empirical evidence (Section~\ref{sec:ablation} and Section~\ref{sec:discussion}) is consistent with this claim.

\subsection{Mixed-precision error propagation across layers}
\label{sec:layer-error}

The per-layer bound in Eq.~\eqref{eq:quant-bound} treats each layer independently.  In practice, KV errors propagate through the residual stream.  We analyze this propagation.

Consider an $L$-layer Transformer.  Let the output of layer $l$ be $x^{(l)} = x^{(l-1)} + a^{(l)}$, where $a^{(l)}$ is the attention output.  If the KV cache at layer $l$ has perturbation $\delta_l$ (from Eq.~\eqref{eq:quant-bound}), the perturbed attention output satisfies $\|\tilde a^{(l)} - a^{(l)}\| \le \delta_l$.  Then:

\begin{proposition}[Layer-wise error propagation]
\label{prop:layer-err}
If each layer $l$ has attention perturbation $\delta_l$ and the layer norm at layer $l$ has Lipschitz constant $\Lambda_l$, then the output perturbation after $L$ layers satisfies:
\begin{equation}
  \|\tilde x^{(L)} - x^{(L)}\| \le \sum_{l=1}^{L} \delta_l \cdot \prod_{k=l+1}^{L} (1 + \Lambda_k \cdot \alpha_k),
  \label{eq:layer-err-prop}
\end{equation}
where $\alpha_k$ is the attention amplification factor at layer $k$, bounded by $e^{2\eta_k}/|S_k|$ from Lemma~\ref{lem:softmax-amp}.
\end{proposition}
\begin{proof}
By induction.  At layer 1: $\|\tilde x^{(1)} - x^{(1)}\| = \|\tilde a^{(1)} - a^{(1)}\| \le \delta_1$.  At layer $l+1$, the input perturbation $\Delta^{(l)}$ propagates through layer norm (bounded by $\Lambda_{l+1}$) and attention (amplified by $\alpha_{l+1}$):
\[
  \|\tilde x^{(l+1)} - x^{(l+1)}\| \le \|\Delta^{(l)}\| \cdot (1 + \Lambda_{l+1} \alpha_{l+1}) + \delta_{l+1}.
\]
Unrolling the recurrence gives Eq.~\eqref{eq:layer-err-prop}.
\end{proof}

The bound has a useful qualitative implication: if early layers have large perturbations (large $\delta_l$ for small $l$), the error is amplified by all subsequent layers.  Conversely, if the perturbation is concentrated in later layers, the amplification factor is smaller.  This suggests that per-layer budget allocation---an extension not implemented in the current \method{}---could further improve quality by favoring higher precision in early layers.

FloatBarrier
\section{Motivation: Locality is Strong but Model-Specific}
\label{sec:motivation}

The sorted assignment of Section~\ref{sec:formulation} presumes a meaningful importance score $I_j$---but do real decode-time attention distributions actually exhibit enough skew to make tiering worthwhile?  The answer is yes, but with an important caveat: the \emph{shape} of the skew varies across model families, and this variation directly determines whether INT4 can safely serve as the low tier.

Measurements use eager attention outputs rather than pre-RoPE hooks, because query/key hooks can miss rotary position embedding and final masks.  We aggregate attention mass over layers, heads, and decode observations and compute the Gini coefficient, active-set ratio, and remote attention mass.

\begin{table}[t]
\centering
\caption{Decode KV locality patterns across model families.}
\label{tab:locality}
\begin{tabularx}{\linewidth}{lclYl}
\toprule
Model & Gini & Pattern & Qualitative behavior & Default tier policy \\
\midrule
Qwen2.5-7B & 0.911 & Local & Attention concentrates near the recent window; sink mass is weaker. & 2-tier, disable \intiv{} \\
Qwen2.5-14B & 0.952 & Local & Even stronger concentration in the locality characterization. & Not fully evaluated \\
Mistral-7B & 0.917 & Sink & A small prefix receives persistent attention; remote mass is sink-shaped. & 3-tier \\
Gemma-2-9B & 0.866 & Hybrid & Sink and local behavior coexist across layers. & 3-tier \\
\bottomrule
\end{tabularx}
\end{table}

\begin{figure}[H]
\centering
\begin{tikzpicture}
\begin{axis}[
    width=0.78\linewidth,
    height=5.2cm,
    ybar,
    ymin=0.80,ymax=0.98,
    ylabel={Gini coefficient},
    symbolic x coords={Qwen7B,Qwen14B,Mistral7B,Gemma9B},
    xtick=data,
    nodes near coords,
    nodes near coords align={vertical},
    bar width=18pt,
    enlarge x limits=0.18,
    grid=both,
    major grid style={draw=gray!25},
    minor tick num=1,
    tick label style={font=\small},
    label style={font=\small}
]
\addplot+[fill=skvblue!45,draw=skvblue] coordinates {(Qwen7B,0.911) (Qwen14B,0.952) (Mistral7B,0.917) (Gemma9B,0.866)};
\end{axis}
\end{tikzpicture}
\caption{KV access Gini coefficients.  All models exhibit strong skew (Gini $>$ 0.86).}
\label{fig:gini}
\end{figure}
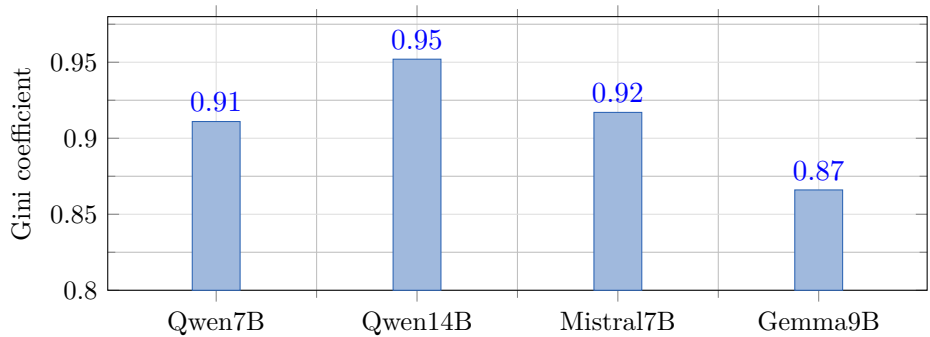

Three observations drive the design.  First, access skew is large across all evaluated models (Figure~\ref{fig:gini}), with a Gini range of 0.866--0.952.  Second, the shape of the skew differs: Qwen is local-dominant, Mistral is sink-dominant, and Gemma is hybrid.  Third, a small hot set can have very high effective capacity: the locality study observed a 130$\times$ hot-set capacity multiplier at 32K context, while the tier-aware adapter overhead was below 0.04\pct{}.  This combination suggests that a small high-fidelity region plus lower-precision cold regions can reduce transfer volume without fully deleting information.  But does this tiered representation preserve enough signal for tasks that depend on rare tokens, such as retrieval?

\section{The \method{} Policy}
\label{sec:method}

\subsection{Overview}

The locality measurements in Section~\ref{sec:motivation} establish two facts that shape the policy design: access skew is large enough to justify tiering, and the pattern of skew differs enough across models to require adaptive tier selection.  Figure~\ref{fig:overview} summarizes how \method{} combines these insights into a concrete pipeline.  The prefill worker computes or approximates token importance, pins protected sink tokens, optionally runs a one-time \intiv{} tolerance probe per model, and emits a precision map.  The map determines how each token's KV tensors are quantized before transfer to the decode worker.

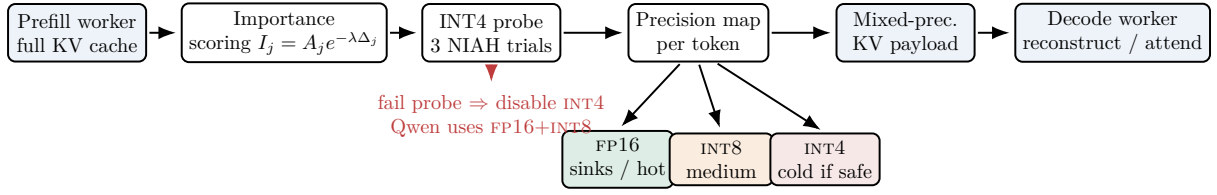
\begin{figure}[H]
\centering
\resizebox{\linewidth}{!}{%
\begin{tikzpicture}[
  scale=0.82, every node/.style={scale=0.82},
  box/.style={rounded corners=3pt, draw, thick, align=center,
    minimum width=2.4cm, minimum height=1.0cm, fill=white, font=\small},
  data/.style={rounded corners=3pt, draw, thick, align=center,
    minimum width=2.4cm, minimum height=1.0cm, fill=skvblue!8, font=\small},
  tier/.style={rounded corners=3pt, draw, thick, align=center,
    minimum width=1.8cm, minimum height=0.85cm, font=\small},
  arr/.style={-{Latex[length=2.5mm]}, thick, shorten >=2pt, shorten <=1pt}
]
  \node[data]                         (prefill)  at (0,0)    {Prefill worker\\full KV cache};
  \node[box]                          (score)    at (3.6,0)  {Importance\\scoring $I_j=A_je^{-\lambda\Delta_j}$};
  \node[box]                          (probe)    at (7.2,0)  {INT4 probe\\3 NIAH trials};
  \node[box]                          (map)      at (10.8,0) {Precision map\\per token};
  \node[data]                         (transfer) at (14.4,0) {Mixed-prec.\\KV payload};
  \node[data]                         (decode)   at (18.0,0) {Decode worker\\reconstruct / attend};

  \node[tier, fill=skvgreen!15]       (fp16)     at (9.4,-2.2) {\fp{}\\sinks / hot};
  \node[tier, fill=skvorange!15]      (int8)     at (11.2,-2.2) {\intviii{}\\medium};
  \node[tier, fill=skvred!12]         (int4)     at (13.0,-2.2) {\intiv{}\\cold if safe};

  \draw[arr] (prefill)  -- (score);
  \draw[arr] (score)    -- (probe);
  \draw[arr] (probe)    -- (map);
  \draw[arr] (map)      -- (transfer);
  \draw[arr] (transfer) -- (decode);

  \draw[arr] (map.240) -- (fp16.north);
  \draw[arr] (map.270) -- (int8.north);
  \draw[arr] (map.300) -- (int4.north);

  \node[below=0.35cm of probe, align=center, font=\small, text=skvred]
    (fallback) {fail probe $\Rightarrow$ disable \intiv{}\\Qwen uses \fp{}+\intviii{}};
  \draw[arr, skvred] (probe.south) -- (fallback.north);
\end{tikzpicture}%
}
\caption{SpectrumKV pipeline overview.}
\label{fig:overview}
\end{figure}

\subsection{System architecture in PD disaggregation}
\label{sec:sysarch}

Figure~\ref{fig:sysarch} places \method{} within a complete PD disaggregated serving system.  The key design decision is that importance scoring and quantization happen at the prefill worker \emph{before} network transfer, while dequantization and attention happen at the decode worker \emph{after} receipt.  This placement ensures that the network payload is already compressed, so only lightweight dequantization (table lookup) is needed at the decode worker before running attention.

\begin{figure}[H]
\centering
\resizebox{\linewidth}{!}{%
\begin{tikzpicture}[
  scale=0.82, every node/.style={scale=0.82},
  worker/.style={rounded corners=4pt, draw, thick, align=center,
    minimum width=2.8cm, minimum height=1.2cm, fill=skvblue!8, font=\small},
  module/.style={rounded corners=3pt, draw, thick, align=center,
    minimum width=2.2cm, minimum height=0.85cm, fill=white, font=\small},
  tierbox/.style={rounded corners=3pt, draw, thick, align=center,
    minimum width=1.4cm, minimum height=0.7cm, font=\small},
  arr/.style={-{Latex[length=2.5mm]}, thick, shorten >=2pt, shorten <=1pt},
  darr/.style={-{Latex[length=2.5mm]}, thick, dashed, skvred, shorten >=2pt, shorten <=1pt}
]
  \node[worker]  (pw)    at (0,0)     {Prefill\\Worker};
  \node[module]  (imp)   at (3.8,0)   {Importance\\Scoring};
  \node[module]  (prb)   at (7.2,0)   {Probe\\Check};
  \node[module]  (quant) at (10.6,0)  {Quantize\\per tier};

  \draw[arr] (pw)     -- (imp);
  \draw[arr] (imp)    -- (prb);
  \draw[arr] (prb)    -- (quant);

  \node[tierbox, fill=skvgreen!15]  (t16) at (14.0, 0.55) {\fp{}};
  \node[tierbox, fill=skvorange!15] (t8)  at (14.0, 0)    {\intviii{}};
  \node[tierbox, fill=skvred!12]    (t4)  at (14.0,-0.55) {\intiv{}};

  \draw[arr] (quant) -- (t16.west);
  \draw[arr] (quant) -- (t8.west);
  \draw[arr] (quant) -- (t4.west);

  \node[font=\small\bfseries] (netlabel) at (16.2,0) {Network};
  \draw[arr, skvblue]   (t16.east) -- (netlabel.west |- t16.east);
  \draw[arr, skvorange] (t8.east)  -- (netlabel.west);
  \draw[arr, skvred]    (t4.east)  -- (netlabel.west |- t4.east);

  \node[worker]  (dw)      at (19.6,0)    {Decode\\Worker};
  \draw[arr, skvblue]   (netlabel.east |- t16.east) -- (dw.west |- t16.east);
  \draw[arr, skvorange] (netlabel.east)           -- (dw.west);
  \draw[arr, skvred]    (netlabel.east |- t4.east) -- (dw.west |- t4.east);

  \node[module]  (dequant) at (19.6,-2.0) {Dequantize\\per tier};
  \node[module]  (attn)    at (16.2,-2.0) {Attention\\compute};

  \draw[arr] (dw) -- (dequant);
  \draw[arr] (dequant) -- (attn);

  \node[below=0.4cm of attn, align=center, font=\small, text=skvred]
    (cold) {Cold-tier fetch\\(upgrade INT4$\to$\fp{})};
  \draw[darr] (attn) -- (cold);
  \draw[darr] (cold.south) -- ++(0,-0.7) -| (quant.south)
    node[pos=0.35, below, font=\footnotesize, text=skvred]{on-demand};
\end{tikzpicture}%
}
\caption{System architecture in PD disaggregation.}
\label{fig:sysarch}
\end{figure}
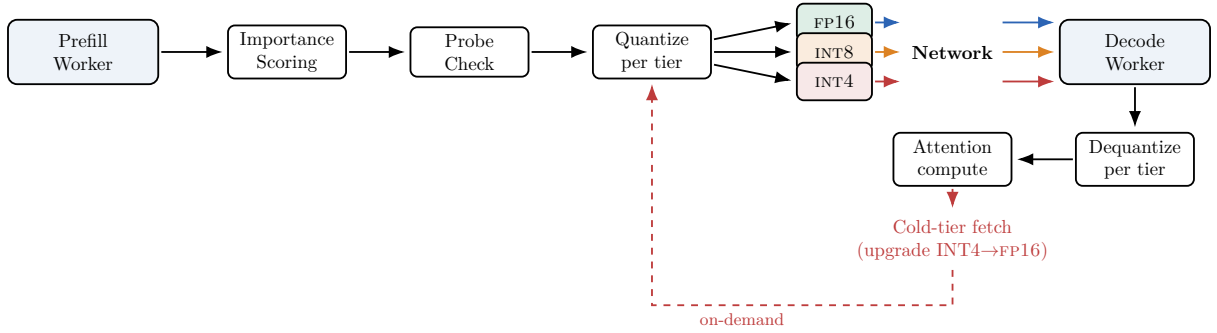

A notable engineering advantage of the mixed-precision paradigm is the design of cold-tier fetches.  In a selection-based system, cold-tier fetch requires transmitting the full FP16 KV for a previously dropped token---a costly operation.  In \method{}, the cold tier already has an approximate (INT4 or INT8) representation at the decode worker.  A cold-tier fetch therefore only needs to \emph{upgrade} the precision from INT4 to FP16, transferring at most $s_{16}-s_4 = 1.75$ bytes per element instead of $s_{16} = 2$ bytes.  More importantly, the approximate representation is already present during the first attention pass, so the model does not suffer from a complete information gap while waiting for the upgrade.

\subsection{Importance scores}
For token position $j$ in a context of length $n$, \method{} uses
\begin{equation}
  I_j = A_j \exp(-\lambda \Delta_j),
  \label{eq:importance}
\end{equation}
where $A_j$ is an observed or approximated attention-importance statistic and $\Delta_j=n-j$ is recency distance.  Sink positions are pinned before ranking.  When attention measurements are unavailable or unreliable, the implementation can fall back to KV-norm scores; in the Qwen decay sweep, this fallback was used and is explicitly reported as such.

\subsection{Budget-to-tier mapping}
For the 3-tier greedy policy without sink pinning, the budget-to-tier mapping follows the normalized cost in Eq.~\eqref{eq:budget}.  At $0.25\le \budget\le 0.5$, no FP16 tokens are needed to meet the budget; the policy uses a mix of \intviii{} and \intiv{}:
\begin{equation}
  f_8 = 4\budget-1,
  \qquad
  f_4 = 2-4\budget,
  \qquad
  f_{16}=0.
\end{equation}
At $0.5\le \budget\le 1$, the policy uses \fp{} and \intviii{}:
\begin{equation}
  f_{16}=2\budget-1,
  \qquad
  f_8=2-2\budget,
  \qquad
  f_4=0.
\end{equation}
SinkProtect reserves the first $s$ positions at \fp{} and applies the same idea to the remaining budget.  For a 2-tier model, \intiv{} is excluded and the lowest tier is \intviii{}.

\subsection{INT4 tolerance probe}
\method{} runs a probe once per model.  The probe performs three \niah{} trials under the aggressive 3-tier greedy configuration at $\budget=0.3$.  If at least two trials succeed, \intiv{} is enabled.  Otherwise the model falls back to 2-tier \fp{}+\intviii{}.  In our experiments, Qwen2.5-7B fails the probe (0/3), while Mistral-7B and Gemma-2-9B pass (3/3 each).  The probe is deliberately task-relevant: it asks whether aggressive low-tier KV still supports retrieval, which is where token dropping and unsafe quantization are most visible.

\begin{center}
\fbox{%
\begin{minipage}{0.92\linewidth}
\textbf{Algorithm 1: \method{} per-token transfer.}
\begin{enumerate}[leftmargin=1.6em,itemsep=1pt]
  \item \textbf{Probe once per model:} run three \niah{} trials with 3-tier greedy at $\budget=0.3$; enable \intiv{} iff at least two trials succeed.
  \item \textbf{Score tokens:} compute $I_j=A_j\exp(-\lambda\Delta_j)$ or a KV-norm fallback; pin the sink prefix to \fp{}.
  \item \textbf{Assign tiers:} sort remaining tokens by $I_j$; allocate \fp{}, \intviii{}, and optionally \intiv{} under the normalized budget.
  \item \textbf{Transfer:} quantize keys and values independently according to the precision map and send the mixed-precision payload to decode.
\end{enumerate}
\end{minipage}}
\end{center}

FloatBarrier
\section{Experimental Setup}
\label{sec:setup}

\paragraph{Models.}
We evaluate Qwen2.5-7B-Instruct, Mistral-7B-Instruct-v0.3, and Gemma-2-9B-it.  The locality characterization also includes Qwen2.5-14B, but the full quality suite is limited to the three 7B--9B models.

\paragraph{Hardware and software.}
Experiments were run on cloud GPUs: NVIDIA RTX 4080 SUPER 32GB and RTX 4090 48GB.  The software stack used PyTorch 2.11.0+cu130, Transformers 5.9.0, and Python 3.12.

\paragraph{Datasets and metrics.}
Perplexity is measured on WikiText-2.  Retrieval is measured with a constructed needle-in-a-haystack benchmark at sequence length 4096 and 19 insertion depths.  Latency is reported as TTFT and decode throughput (tokens/s) for the transfer path.  We report relative PPL change against full FP16 KV:
\begin{equation}
  \Delta \ppl = 100\cdot\frac{\ppl_{\mathrm{method}}-\ppl_{\mathrm{FP16}}}{\ppl_{\mathrm{FP16}}}.
\end{equation}
A negative value is not interpreted as an intrinsic quality improvement; it means the measured PPL is slightly below the full-KV run under the same protocol.

\paragraph{Baselines.}
PDTrim represents fixed first/last token selection for PD transfer.  SWS (Semantic Working Set) is a selection-based policy that drops unselected tokens entirely.  Uniform$_{\text{INT8}}$ quantizes all tokens to \intviii{}.  Random$_{\text{Tier}}$ assigns the same tier fractions without importance ranking.  We also compare \method{} variants: Greedy, SinkProtect, Balanced, and Adaptive.

\paragraph{Data hygiene.}
The final measurements use corrected implementations.  Important fixes include using eager attention for locality characterization, replacing zero-out baselines with attention-mask deletion for selection methods, adding the missing \niah{} retrieval question, fixing an adaptive-policy control-flow bug that could expose dropped tokens, and applying the Gemma chat template.  These fixes matter: earlier intermediate logs were marked obsolete and are not used for the main claims.

FloatBarrier
\section{Results}
\label{sec:results}

\subsection{PPL at the 50\pct{} transfer budget}
\label{sec:rq7}

Table~\ref{tab:ppl50} is the central PPL comparison.  At $\budget=0.5$, \method{} Greedy is equivalent to all-token \intviii{} transfer, which explains why Greedy, Uniform$_{\text{INT8}}$, and Random$_{\text{Tier}}$ have the same PPL.  The key result is that preserving all tokens at reduced precision is substantially better than dropping half the tokens: PDTrim degrades PPL by +25.85\pct{} on Qwen, +22.07\pct{} on Mistral, and +35.63\pct{} on Gemma, while \method{} Greedy changes PPL by only +1.97\pct{}, $-$0.06\pct{}, and $-$0.44\pct{}, respectively.  SinkProtect slightly improves Qwen and Gemma at the same budget.

\begin{table}[t]
\centering
\caption{PPL comparison at $\budget=0.5$, seq=2048.}
\label{tab:ppl50}
\small
\begin{tabular}{lrrr}
\toprule
Method & Qwen2.5-7B & Mistral-7B & Gemma-2-9B \\
\midrule
Uniform$_{\text{INT8}}$ & +1.97 & $-$0.06 & $-$0.44 \\
PDTrim & +25.85 & +22.07 & +35.63 \\
SWS$_{\text{Original}}$ & +30.88 & +33.25 & +43.78 \\
SWS$_{\text{ValueAware}}$ & +30.88 & +33.25 & +43.78 \\
\method{}$_{\text{Greedy}}$ & +1.97 & $-$0.06 & $-$0.44 \\
\method{}$_{\text{SinkProtect}}$ & \textbf{+1.39} & \textbf{$-$0.10} & \textbf{$-$0.47} \\
\method{}$_{\text{Balanced}}$ & +7506.84 & +2.71 & +2.30 \\
Random$_{\text{Tier}}$ & +1.97 & $-$0.06 & $-$0.44 \\
\bottomrule
\end{tabular}
\end{table}

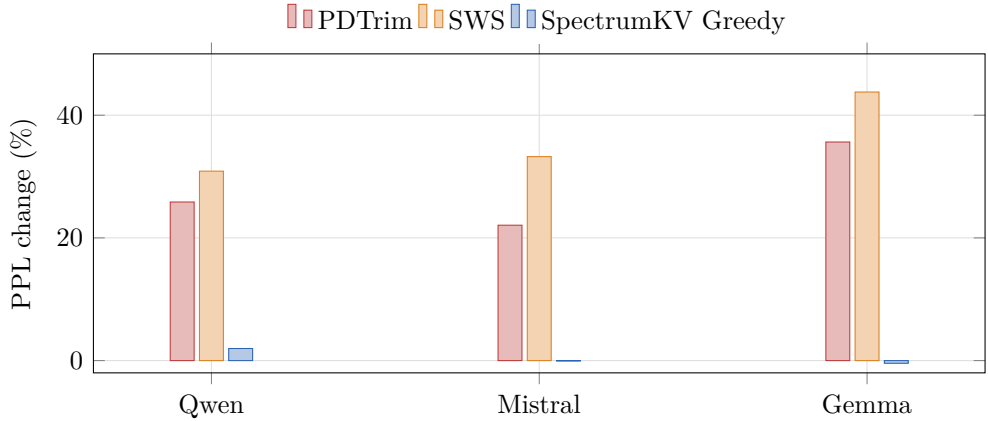
\begin{figure}[H]
\centering
\begin{tikzpicture}
\begin{axis}[
    width=0.84\linewidth,
    height=5.8cm,
    ybar,
    bar width=9pt,
    ymin=-2,ymax=50,
    ylabel={PPL change (\%)},
    symbolic x coords={Qwen,Mistral,Gemma},
    xtick=data,
    enlarge x limits=0.18,
    legend style={at={(0.5,1.03)},anchor=south,legend columns=3,font=\small,draw=none},
    grid=major,
    major grid style={draw=gray!25},
    tick label style={font=\small},
    label style={font=\small}
]
\addplot+[fill=skvred!35,draw=skvred] coordinates {(Qwen,25.85) (Mistral,22.07) (Gemma,35.63)};
\addplot+[fill=skvorange!35,draw=skvorange] coordinates {(Qwen,30.88) (Mistral,33.25) (Gemma,43.78)};
\addplot+[fill=skvblue!35,draw=skvblue] coordinates {(Qwen,1.97) (Mistral,-0.06) (Gemma,-0.44)};
\legend{PDTrim,SWS,\method{} Greedy}
\end{axis}
\end{tikzpicture}
\caption{Mixed-precision vs.\ binary selection at $\budget=0.5$.}
\label{fig:ppl50}
\end{figure}

The Balanced row is intentionally not hidden.  It assigns substantial \intiv{} mass even at $\budget=0.5$ and causes Qwen PPL to explode.  This is not a plotting artifact or a malformed cache write; it is reproduced by direct \intiv{} controls and motivates the probe.

\subsection{Budget sweep}
\label{sec:rq8}

Table~\ref{tab:budget} shows that the advantage grows at lower budgets.  For Mistral and Gemma, 3-tier \method{} remains usable at $\budget=0.3$ (+5.41\pct{} and +2.45\pct{} PPL), while PDTrim and SWS degrade much more.  Qwen is marked unsafe for 3-tier at $\budget=0.3$ because the assignment places roughly 80\pct{} of tokens in \intiv{}, producing catastrophic PPL.  The correct Qwen policy at this budget is adaptive 2-tier; we evaluate it primarily through retrieval because the main PPL scan focused on fixed 3-tier variants.

\begin{table}[t]
\centering
\caption{Budget sweep at seq=2048.  Values are PPL change vs.\ full FP16 KV.}
\label{tab:budget}
\small
\begin{tabular}{llrrr}
\toprule
Model & Method & $\budget=0.3$ & $\budget=0.5$ & $\budget=0.7$ \\
\midrule
\multirow{3}{*}{Qwen2.5-7B}
  & PDTrim & +42.82 & +25.85 & +18.14 \\
  & SWS & +56.76 & +30.88 & +15.73 \\
  & \method{} Greedy & unsafe & +1.97 & +0.47 \\
\midrule
\multirow{3}{*}{Mistral-7B}
  & PDTrim & +30.90 & +22.07 & +14.36 \\
  & SWS & +48.29 & +33.25 & +19.45 \\
  & \method{} Greedy & +5.41 & $-$0.06 & +0.03 \\
\midrule
\multirow{3}{*}{Gemma-2-9B}
  & PDTrim & +81.74 & +35.63 & +15.31 \\
  & SWS & +101.28 & +43.78 & +22.48 \\
  & \method{} Greedy & +2.45 & $-$0.44 & $-$0.03 \\
\bottomrule
\end{tabular}
\end{table}

\subsection{PPL vs.\ budget curves}
\label{sec:ppl-curves}

Figures~\ref{fig:ppl-mistral}--\ref{fig:ppl-qwen} present the full PPL-vs-budget curves across all 11 budget points for each model.  The curves make the quality--budget tradeoff visible across the entire operating range.

For Mistral and Gemma, the \method{}$_{\text{Greedy}}$ and Random$_{\text{Tier}}$ curves lie far below PDTrim and SWS across all budgets, and the improvement is largest at low budgets where INT4 tokens carry useful information that would otherwise be lost entirely.  For Qwen, the safe operating range starts at $\budget=0.5$: below this threshold, the 3-tier policy assigns INT4 to a large fraction of tokens and quality collapses.  The shaded unsafe zone in Figure~\ref{fig:ppl-qwen} marks this region.

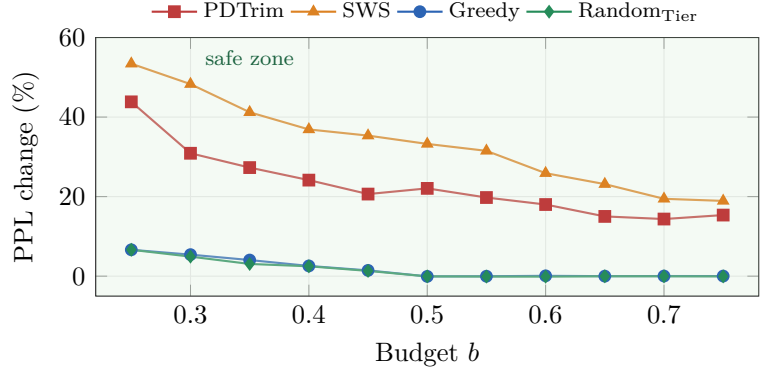
\begin{figure}[H]
\centering
\begin{tikzpicture}
\begin{axis}[
    name=mistral,
    width=0.65\linewidth,
    height=5cm,
    xlabel={Budget $\budget$},
    ylabel={PPL change (\%)},
    xmin=0.22,xmax=0.78,
    ymin=-5,ymax=60,
    legend style={at={(0.5,1.02)},anchor=south,legend columns=4,font=\scriptsize,draw=none,row sep=0pt},
    grid=major,
    major grid style={draw=gray!20},
    tick label style={font=\small},
    label style={font=\small},
    title={\small Mistral-7B},
    title style={font=\small\bfseries}
]
\addplot[thick,skvred,mark=square*,mark size=2pt] coordinates {
    (0.25,43.83)(0.3,30.90)(0.35,27.30)(0.4,24.15)(0.45,20.64)
    (0.5,22.07)(0.55,19.77)(0.6,18.01)(0.65,15.02)(0.7,14.36)(0.75,15.36)
};
\addplot[thick,skvorange,mark=triangle*,mark size=2pt] coordinates {
    (0.25,53.42)(0.3,48.29)(0.35,41.21)(0.4,36.89)(0.45,35.34)
    (0.5,33.25)(0.55,31.51)(0.6,25.89)(0.65,23.14)(0.7,19.45)(0.75,18.92)
};
\addplot[thick,skvblue,mark=*,mark size=2pt] coordinates {
    (0.25,6.63)(0.3,5.41)(0.35,4.05)(0.4,2.57)(0.45,1.45)
    (0.5,-0.06)(0.55,-0.02)(0.6,0.08)(0.65,0.00)(0.7,0.03)(0.75,0.02)
};
\addplot[thick,skvgreen,mark=diamond*,mark size=2pt] coordinates {
    (0.25,6.63)(0.3,4.91)(0.35,3.07)(0.4,2.48)(0.45,1.32)
    (0.5,-0.06)(0.55,-0.10)(0.6,-0.12)(0.65,-0.03)(0.7,-0.04)(0.75,-0.06)
};
\legend{PDTrim,SWS,Greedy,Random$_{\text{Tier}}$}
\fill[safezone,opacity=0.3] (axis cs:0.22,-5) rectangle (axis cs:0.78,60);
\node[font=\scriptsize,text=skvgreen!70!black] at (axis cs:0.35,55) {safe zone};
\end{axis}
\end{tikzpicture}
\caption{PPL change vs.\ KV budget for Mistral-7B.}
\label{fig:ppl-mistral}
\end{figure}

As shown in Figure~\ref{fig:ppl-mistral}, Mistral-7B tolerates the full budget range under the 3-tier policy.  SpectrumKV$_{\text{Greedy}}$ and Random$_{\text{Tier}}$ remain within $\pm0.1\%$ of the FP16 baseline for $\budget\ge0.5$, while PDTrim and SWS exceed $+20\%$ at the same budget.  At $\budget=0.3$, Greedy reaches $+5.41\%$---a non-trivial increase, but far below the $+30.9\%$ degradation of PDTrim and $+48.3\%$ of SWS.  The near-overlap of Greedy and Random$_{\text{Tier}}$ for $\budget\le0.3$ reflects the fact that, when most tokens are INT4, importance ranking has limited room to differentiate.

\begin{figure}[H]
\centering
\begin{tikzpicture}
\begin{axis}[
    name=gemma,
    width=0.65\linewidth,
    height=5cm,
    xlabel={Budget $\budget$},
    ylabel={PPL change (\%)},
    xmin=0.22,xmax=0.78,
    ymin=-5,ymax=140,
    legend style={at={(0.5,1.02)},anchor=south,legend columns=4,font=\scriptsize,draw=none,row sep=0pt},
    grid=major,
    major grid style={draw=gray!20},
    tick label style={font=\small},
    label style={font=\small},
    title={\small Gemma-2-9B},
    title style={font=\small\bfseries}
]
\addplot[thick,skvred,mark=square*,mark size=2pt] coordinates {
    (0.25,98.52)(0.3,81.74)(0.35,53.64)(0.4,49.91)(0.45,35.44)
    (0.5,35.63)(0.55,28.27)(0.6,28.47)(0.65,21.75)(0.7,15.31)(0.75,17.70)
};
\addplot[thick,skvorange,mark=triangle*,mark size=2pt] coordinates {
    (0.25,134.50)(0.3,101.28)(0.35,88.12)(0.4,70.37)(0.45,56.03)
    (0.5,43.78)(0.55,36.81)(0.6,32.84)(0.65,29.51)(0.7,22.48)(0.75,17.59)
};
\addplot[thick,skvblue,mark=*,mark size=2pt] coordinates {
    (0.25,4.74)(0.3,2.45)(0.35,1.53)(0.4,0.52)(0.45,0.07)
    (0.5,-0.44)(0.55,-0.12)(0.6,0.15)(0.65,-0.04)(0.7,-0.03)(0.75,0.04)
};
\addplot[thick,skvgreen,mark=diamond*,mark size=2pt] coordinates {
    (0.25,4.74)(0.3,2.52)(0.35,1.87)(0.4,0.92)(0.45,0.55)
    (0.5,-0.44)(0.55,-0.75)(0.6,0.17)(0.65,-0.08)(0.7,-0.01)(0.75,-0.00)
};
\legend{PDTrim,SWS,Greedy,Random$_{\text{Tier}}$}
\fill[safezone,opacity=0.3] (axis cs:0.22,-5) rectangle (axis cs:0.78,140);
\node[font=\scriptsize,text=skvgreen!70!black] at (axis cs:0.35,130) {safe zone};
\end{axis}
\end{tikzpicture}
\caption{PPL change vs.\ KV budget for Gemma-2-9B.}
\label{fig:ppl-gemma}
\end{figure}
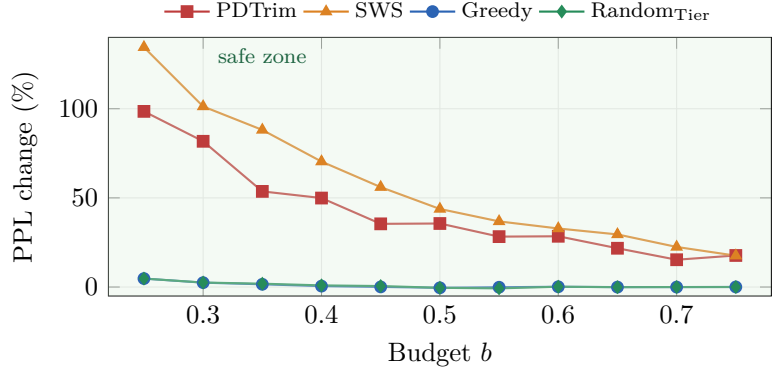

Figure~\ref{fig:ppl-gemma} shows a qualitatively similar pattern for Gemma-2-9B, but with a larger absolute spread: SWS exceeds $+100\%$ PPL change at $\budget=0.25$, while SpectrumKV$_{\text{Greedy}}$ stays below $+5\%$ across the entire range.  The wider y-axis scale reflects Gemma's higher sensitivity to token dropping; the mixed-precision approach avoids this by retaining every token, albeit at reduced precision.  The gap between Greedy and Random$_{\text{Tier}}$ is slightly larger than for Mistral, suggesting that importance ranking provides more benefit when the model's attention distribution is hybrid (sink + local) rather than predominantly sink-shaped.

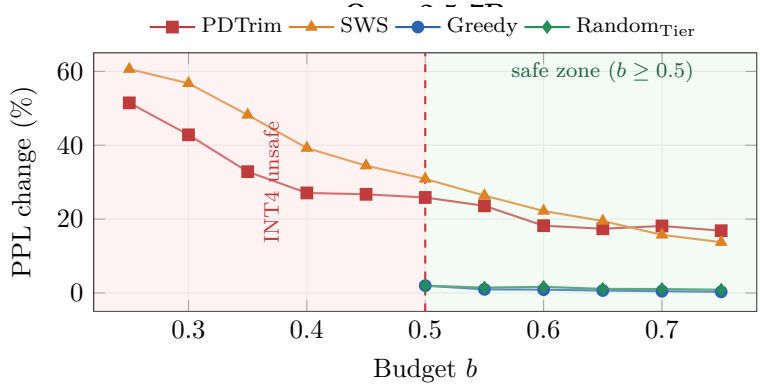
\begin{figure}[H]
\centering
\begin{tikzpicture}
\begin{axis}[
    name=qwen,
    width=0.65\linewidth,
    height=5cm,
    xlabel={Budget $\budget$},
    ylabel={PPL change (\%)},
    xmin=0.22,xmax=0.78,
    ymin=-5,ymax=65,
    legend style={at={(0.5,1.02)},anchor=south,legend columns=4,font=\scriptsize,draw=none,row sep=0pt},
    grid=major,
    major grid style={draw=gray!20},
    tick label style={font=\small},
    label style={font=\small},
    title={\small Qwen2.5-7B},
    title style={font=\small\bfseries}
]
\addplot[thick,skvred,mark=square*,mark size=2pt] coordinates {
    (0.25,51.47)(0.3,42.82)(0.35,32.84)(0.4,27.08)(0.45,26.71)
    (0.5,25.85)(0.55,23.59)(0.6,18.21)(0.65,17.38)(0.7,18.14)(0.75,16.85)
};
\addplot[thick,skvorange,mark=triangle*,mark size=2pt] coordinates {
    (0.25,60.61)(0.3,56.76)(0.35,48.21)(0.4,39.23)(0.45,34.47)
    (0.5,30.88)(0.55,26.37)(0.6,22.19)(0.65,19.48)(0.7,15.73)(0.75,13.74)
};
\addplot[thick,skvblue,mark=*,mark size=2pt] coordinates {
    (0.5,1.97)(0.55,0.97)(0.6,0.88)(0.65,0.65)(0.7,0.47)(0.75,0.30)
};
\addplot[thick,skvgreen,mark=diamond*,mark size=2pt] coordinates {
    (0.5,1.97)(0.55,1.44)(0.6,1.65)(0.65,1.10)(0.7,1.05)(0.75,0.87)
};
\legend{PDTrim,SWS,Greedy,Random$_{\text{Tier}}$}
\fill[unsafezone,opacity=0.35] (axis cs:0.22,-5) rectangle (axis cs:0.50,65);
\fill[safezone,opacity=0.3] (axis cs:0.50,-5) rectangle (axis cs:0.78,65);
\draw[dashed,thick,skvred] (axis cs:0.50,-5) -- (axis cs:0.50,65);
\node[font=\scriptsize,text=skvred,rotate=90] at (axis cs:0.37,30) {INT4 unsafe};
\node[font=\scriptsize,text=skvgreen!70!black] at (axis cs:0.65,60) {safe zone ($\budget\ge0.5$)};
\end{axis}
\end{tikzpicture}
\caption{PPL change vs.\ KV budget for Qwen2.5-7B.  Shaded region: INT4 unsafe zone.}
\label{fig:ppl-qwen}
\end{figure}

An interesting observation is that Random$_{\text{Tier}}$ and \method{}$_{\text{Greedy}}$ converge at $\budget=0.5$ for all models.  This is expected: at $\budget=0.5$, all tokens are assigned INT8 under both policies, making importance ranking irrelevant.  Below $\budget=0.5$, the two diverge, but not always in the expected direction---for Mistral at $\budget=0.25$, Random$_{\text{Tier}}$ matches Greedy, and at $\budget=0.3$, Random$_{\text{Tier}}$ (4.91\pct{}) is slightly better than Greedy (5.41\pct{}).  We discuss this counterintuitive result in Section~\ref{sec:ablation-ranking}.

\subsection{Retrieval preservation}
\label{sec:rq11}

Retrieval is where the difference between approximation and deletion is most visible.  A pruned needle cannot be recovered by the decode worker unless the system fetches the missing KV.  A low-precision needle, by contrast, still contributes an approximate signal.

\begin{table}[t]
\centering
\caption{NIAH retrieval accuracy at seq=4096.}
\label{tab:niah}
\small
\begin{tabular}{llrrrr}
\toprule
Model & Method & $\budget=0.3$ & $\budget=0.4$ & $\budget=0.5$ & $\budget=0.7$ \\
\midrule
\multirow{4}{*}{Qwen2.5-7B}
  & PDTrim & 26.3 & 36.8 & 47.4 & 68.4 \\
  & SWS & 21.1 & 31.6 & 43.2 & 68.4 \\
  & \method{} Greedy 3-tier & 4.2 & 0.0 & 100.0 & 100.0 \\
  & \method{} Adaptive 2-tier & \textbf{52.6} & \textbf{74.7} & \textbf{100.0} & \textbf{100.0} \\
\midrule
\multirow{4}{*}{Mistral-7B}
  & PDTrim & 26.3 & 34.7 & 44.2 & 62.1 \\
  & SWS & 20.0 & 29.5 & 37.9 & 61.1 \\
  & \method{} Greedy raw & 89.5 & 89.5 & 89.5 & 89.5 \\
  & \method{} Adaptive verification & \textbf{100.0} & \textbf{100.0} & \textbf{100.0} & \textbf{100.0} \\
\midrule
\multirow{3}{*}{Gemma-2-9B}
  & PDTrim & 41.1 & 49.5 & 57.9 & 74.7 \\
  & SWS & 35.8 & 44.2 & 52.6 & 72.6 \\
  & \method{} Greedy/SinkProtect & \textbf{100.0} & \textbf{100.0} & \textbf{100.0} & \textbf{100.0} \\
\bottomrule
\end{tabular}
\end{table}

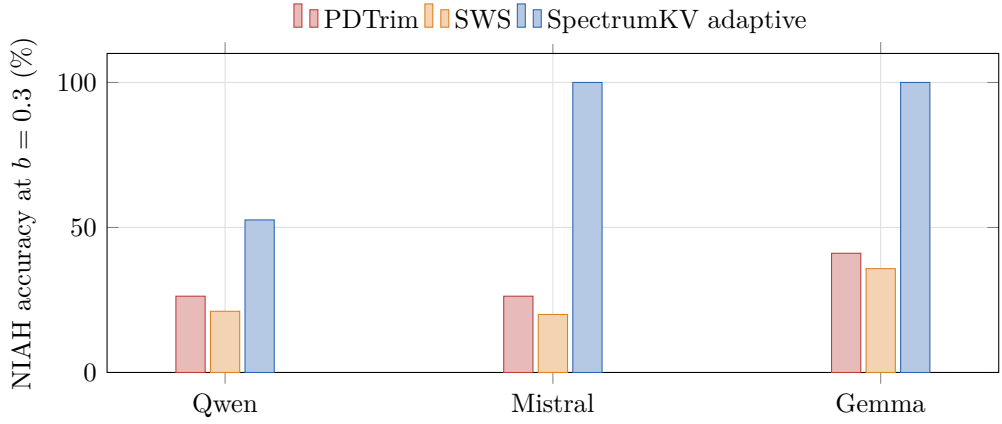
\begin{figure}[H]
\centering
\begin{tikzpicture}
\begin{axis}[
    width=0.84\linewidth,
    height=5.8cm,
    ybar,
    bar width=11pt,
    ymin=0,ymax=110,
    ylabel={NIAH accuracy at $\budget=0.3$ (\%)},
    symbolic x coords={Qwen,Mistral,Gemma},
    xtick=data,
    enlarge x limits=0.18,
    legend style={at={(0.5,1.03)},anchor=south,legend columns=3,font=\small,draw=none},
    grid=major,
    major grid style={draw=gray!25},
    tick label style={font=\small},
    label style={font=\small}
]
\addplot+[fill=skvred!35,draw=skvred] coordinates {(Qwen,26.3) (Mistral,26.3) (Gemma,41.1)};
\addplot+[fill=skvorange!35,draw=skvorange] coordinates {(Qwen,21.1) (Mistral,20.0) (Gemma,35.8)};
\addplot+[fill=skvblue!35,draw=skvblue] coordinates {(Qwen,52.6) (Mistral,100.0) (Gemma,100.0)};
\legend{PDTrim,SWS,\method{} adaptive}
\end{axis}
\end{tikzpicture}
\caption{NIAH accuracy at $\budget=0.3$.}
\label{fig:niah03}
\end{figure}

\subsection{NIAH depth--budget heatmap}
\label{sec:niah-heatmap}

The aggregate \niah{} accuracy in Table~\ref{tab:niah} obscures the depth-dependent structure of retrieval success.  Figure~\ref{fig:niah-heatmap} shows the full depth-vs-budget retrieval map for Qwen2.5-7B, comparing \method{}$_{\text{Adaptive}}$ with PDTrim.

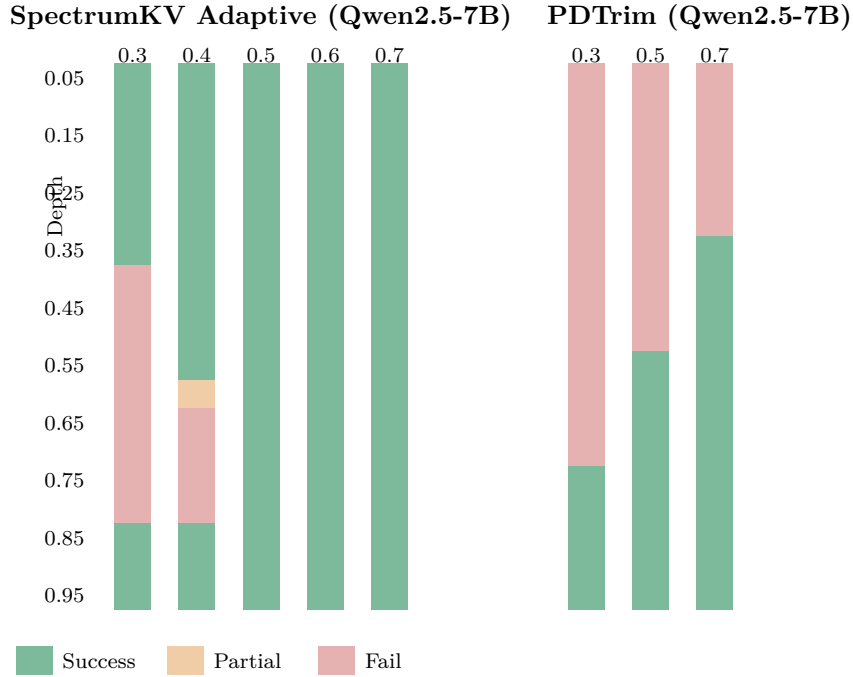
\begin{figure}[H]
\centering
\begin{tikzpicture}[
  cell/.style={minimum width=0.48cm,minimum height=0.38cm,inner sep=0pt,outer sep=0pt},
  succ/.style={cell,fill=skvgreen!60},
  fail/.style={cell,fill=skvred!40},
  part/.style={cell,fill=skvorange!40},
  lbl/.style={font=\scriptsize}
]

\node[font=\small\bfseries] at (3.2,4.7) {\method{} Adaptive (Qwen2.5-7B)};

\def\depths{{0.05,0.10,0.15,0.20,0.25,0.30,0.35,0.40,0.45,0.50,0.55,0.60,0.65,0.70,0.75,0.80,0.85,0.90,0.95}}
\def\budgets{{0.3,0.4,0.5,0.6,0.7}}


\foreach \bi/\blbl in {0/0.3,1/0.4,2/0.5,3/0.6,4/0.7} {
  \node[lbl] at (1.5+\bi*0.85,4.2) {\blbl};
}
\node[lbl,rotate=90] at (0.5,2.2) {Depth};

\foreach \di/\dlbl in {0/0.05,1/0.10,2/0.15,3/0.20,4/0.25,5/0.30,6/0.35,7/0.40,8/0.45,9/0.50,10/0.55,11/0.60,12/0.65,13/0.70,14/0.75,15/0.80,16/0.85,17/0.90,18/0.95} {
  \pgfmathparse{mod(\di,2)==0 ? 1 : 0}
  \ifnum\pgfmathresult=1
    \node[lbl,anchor=east] at (1.0,3.9-\di*0.38) {\dlbl};
  \fi
}

\foreach \di in {0,...,6} { \node[succ] at (1.5+0*0.85,3.9-\di*0.38) {}; }
\foreach \di in {7,...,15} { \node[fail] at (1.5+0*0.85,3.9-\di*0.38) {}; }
\foreach \di in {16,...,18} { \node[succ] at (1.5+0*0.85,3.9-\di*0.38) {}; }

\foreach \di in {0,...,10} { \node[succ] at (1.5+1*0.85,3.9-\di*0.38) {}; }
\node[part] at (1.5+1*0.85,3.9-11*0.38) {};
\foreach \di in {12,...,15} { \node[fail] at (1.5+1*0.85,3.9-\di*0.38) {}; }
\foreach \di in {16,...,18} { \node[succ] at (1.5+1*0.85,3.9-\di*0.38) {}; }

\foreach \di in {0,...,18} { \node[succ] at (1.5+2*0.85,3.9-\di*0.38) {}; }

\foreach \di in {0,...,18} { \node[succ] at (1.5+3*0.85,3.9-\di*0.38) {}; }

\foreach \di in {0,...,18} { \node[succ] at (1.5+4*0.85,3.9-\di*0.38) {}; }

\node[font=\small\bfseries] at (9.0,4.7) {PDTrim (Qwen2.5-7B)};

\foreach \bi/\blbl in {0/0.3,1/0.5,2/0.7} {
  \node[lbl] at (7.5+\bi*0.85,4.2) {\blbl};
}

\foreach \di in {0,...,13} { \node[fail] at (7.5+0*0.85,3.9-\di*0.38) {}; }
\foreach \di in {14,...,18} { \node[succ] at (7.5+0*0.85,3.9-\di*0.38) {}; }

\foreach \di in {0,...,9} { \node[fail] at (7.5+1*0.85,3.9-\di*0.38) {}; }
\foreach \di in {10,...,18} { \node[succ] at (7.5+1*0.85,3.9-\di*0.38) {}; }

\foreach \di in {0,...,5} { \node[fail] at (7.5+2*0.85,3.9-\di*0.38) {}; }
\foreach \di in {6,...,18} { \node[succ] at (7.5+2*0.85,3.9-\di*0.38) {}; }

\node[succ,label={right:\scriptsize Success}] at (0.2,-3.8) {};
\node[part,label={right:\scriptsize Partial}] at (2.2,-3.8) {};
\node[fail,label={right:\scriptsize Fail}] at (4.2,-3.8) {};

\end{tikzpicture}
\caption{NIAH depth--budget heatmap for Qwen2.5-7B.  Left: SpectrumKV Adaptive (2-tier).  Right: PDTrim.}
\label{fig:niah-heatmap}
\end{figure}

The heatmap reveals a characteristic U-shaped retrieval pattern for \method{} at low budgets.  At $\budget=0.3$, Adaptive successfully retrieves needles at shallow depths ($d \le 0.35$, near the attention sink region) and at deep depths ($d \ge 0.85$, near the local window), while failing in the middle range ($0.40 \le d \le 0.80$).  This U-shape reflects the importance scoring: sink tokens and recent tokens receive FP16 or INT8, while middle-context tokens receive INT8 but the budget is insufficient to cover all positions with adequate precision.

PDTrim, by contrast, shows a monotone depth threshold: retrieval succeeds only when the needle is in the retained suffix.  At $\budget=0.3$, only depths $d \ge 0.75$ are retrieved.  The U-shape is absent because PDTrim simply deletes the early context entirely; there is no approximate representation to exploit.

This pattern has a practical implication: for workloads where the answer is known to be near the beginning or end of the context (common in document QA), \method{} can operate at much lower budgets than selection methods.

\subsection{Latency evidence}
\label{sec:rq9}

Table~\ref{tab:latency} reports the GPU timing experiment.  At $\budget=0.5$, TTFT falls by 50--62\pct{} across all tested model/length pairs.  Decode throughput is mostly stable and can improve at longer lengths, where reducing KV traffic relieves memory-bandwidth pressure.

\begin{table}[t]
\centering
\caption{End-to-end TTFT and decode throughput at $\budget=0.5$.}
\label{tab:latency}
\small
\begin{tabular}{llrrrrrr}
\toprule
Model & Seq & Full TTFT & $\budget=0.5$ TTFT & TTFT $\downarrow$ & Full TPS & $\budget=0.5$ TPS & TPS $\Delta$ \\
\midrule
Qwen2.5-7B & 2K & 391 & 154 & $-$61\pct{} & 49.2 & 43.5 & $-$11\pct{} \\
Qwen2.5-7B & 4K & 630 & 296 & $-$53\pct{} & 42.4 & 43.2 & +2\pct{} \\
Mistral-7B & 2K & 274 & 116 & $-$58\pct{} & 51.9 & 50.1 & $-$3\pct{} \\
Mistral-7B & 4K & 497 & 233 & $-$53\pct{} & 47.1 & 49.1 & +4\pct{} \\
Mistral-7B & 8K & 1211 & 502 & $-$59\pct{} & 44.1 & 47.0 & +7\pct{} \\
Gemma-2-9B & 2K & 269 & 135 & $-$50\pct{} & 38.2 & 38.9 & +2\pct{} \\
Gemma-2-9B & 4K & 709 & 270 & $-$62\pct{} & 28.8 & 38.2 & +33\pct{} \\
Gemma-2-9B & 8K & 1664 & 713 & $-$57\pct{} & 27.3 & 28.8 & +5\pct{} \\
\bottomrule
\end{tabular}
\end{table}

\begin{figure}[H]
\centering
\begin{tikzpicture}
\begin{axis}[
    width=0.92\linewidth,
    height=5.2cm,
    ybar,
    ymin=0,ymax=70,
    ylabel={TTFT reduction (\%)},
    symbolic x coords={Q2K,Q4K,M2K,M4K,M8K,G2K,G4K,G8K},
    xtick=data,
    bar width=12pt,
    enlarge x limits=0.07,
    grid=major,
    major grid style={draw=gray!25},
    tick label style={font=\small},
    label style={font=\small},
    nodes near coords,
    nodes near coords align={vertical}
]
\addplot+[fill=skvgreen!35,draw=skvgreen] coordinates {(Q2K,61) (Q4K,53) (M2K,58) (M4K,53) (M8K,59) (G2K,50) (G4K,62) (G8K,57)};
\end{axis}
\end{tikzpicture}
\caption{TTFT reduction at $\budget=0.5$.}
\label{fig:ttft}
\end{figure}
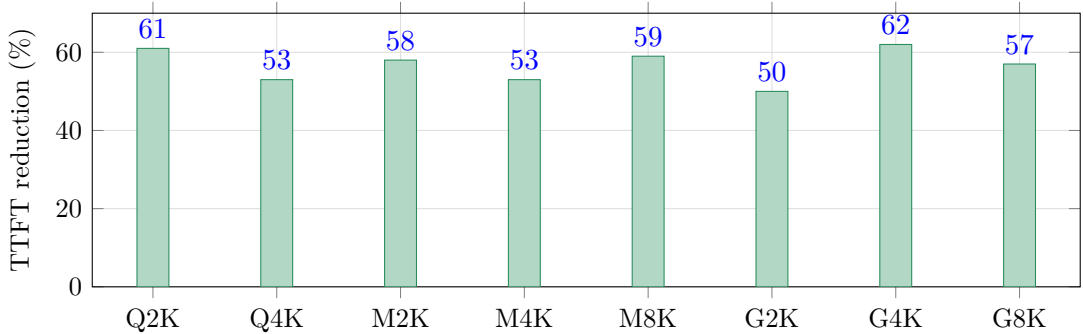

These timings should be read as transfer-path evidence rather than a full production serving study.  They do not include a complete PD scheduler, continuous batching, network contention, or cold-tier fetch policies.

\subsection{Context length}
\label{sec:rq10}

Table~\ref{tab:context} uses the raw GPU PPL scan.  \method{} remains close to baseline across 2K--8K contexts.  Mistral is especially stable, with absolute PPL deltas within 0.3\pct{} for Greedy.  Gemma 8K was not completed because of OOM risk in the available setup.

\begin{table}[t]
\centering
\caption{Context-length sweep at $\budget=0.5$.}
\label{tab:context}
\small
\resizebox{\linewidth}{!}{%
\begin{tabular}{llrrrr}
\toprule
Model & Seq & Full PPL & \method{} Greedy $\Delta$ & \method{} SinkProtect $\Delta$ & PDTrim $\Delta$ \\
\midrule
Qwen2.5-7B & 2048 & 7.0415 & +1.97 & +1.39 & +25.85 \\
Qwen2.5-7B & 4096 & 6.0084 & +2.45 & +1.76 & +16.48 \\
Qwen2.5-7B & 8192 & 6.2145 & +2.90 & +2.54 & +7.24 \\
\midrule
Mistral-7B & 2048 & 7.9568 & $-$0.06 & $-$0.10 & +22.07 \\
Mistral-7B & 4096 & 6.2379 & $-$0.13 & +0.03 & +13.69 \\
Mistral-7B & 8192 & 6.4185 & $-$0.29 & $-$0.08 & +7.95 \\
\midrule
Gemma-2-9B & 2048 & 11.2042 & $-$0.44 & $-$0.47 & +35.63 \\
Gemma-2-9B & 4096 & 7.9672 & +0.14 & $-$0.20 & +9.31 \\
Gemma-2-9B & 8192 & --- & --- & --- & --- \\
\bottomrule
\end{tabular}%
}
\end{table}

\subsection{Probe and quantization error}
\label{sec:probequant}

Table~\ref{tab:probe} shows the probe decisions.  The probe is deliberately small, but it catches the only catastrophic \intiv{} failure in the evaluated set.  Table~\ref{tab:cosine} then shows why a separate probe is necessary: Qwen's average \intiv{} key cosine similarity is \emph{higher} than Mistral's and Gemma's, yet Qwen fails.  A vector-level cosine statistic is not enough to predict how key perturbations propagate through a low-entropy attention distribution.

\begin{table}[t]
\centering
\caption{INT4 tolerance probe and adaptive tier decision.}
\label{tab:probe}
\small
\begin{tabular}{lccc}
\toprule
Model & Probe successes & INT4 tolerant? & Assigned policy \\
\midrule
Qwen2.5-7B & 0/3 & No & 2-tier (\fp{}+\intviii{}) \\
Mistral-7B & 3/3 & Yes & 3-tier (\fp{}+\intviii{}+\intiv{}) \\
Gemma-2-9B & 3/3 & Yes & 3-tier (\fp{}+\intviii{}+\intiv{}) \\
\bottomrule
\end{tabular}
\end{table}

\begin{table}[t]
\centering
\caption{Cosine similarity between quantized and FP16 KV tensors.}
\label{tab:cosine}
\small
\begin{tabular}{llrrr}
\toprule
Quantization & Tensor & Qwen2.5-7B & Mistral-7B & Gemma-2-9B \\
\midrule
\intviii{} & Key & $>$0.9999 & $>$0.9999 & $>$0.9999 \\
\intviii{} & Value & $>$0.9999 & $>$0.9999 & $>$0.9999 \\
\intiv{} & Key & 0.9770 & 0.9792 & 0.9636 \\
\intiv{} & Value & 0.9895 & 0.9906 & 0.9844 \\
\bottomrule
\end{tabular}
\end{table}

\subsection{Per-layer quantization error}
\label{sec:per-layer-quant}

The aggregate cosine similarity in Table~\ref{tab:cosine} masks significant per-layer variation.  Figure~\ref{fig:cosine-qwen} shows the INT4 key and value cosine similarity for each layer of the three models, extracted from the experiment data.

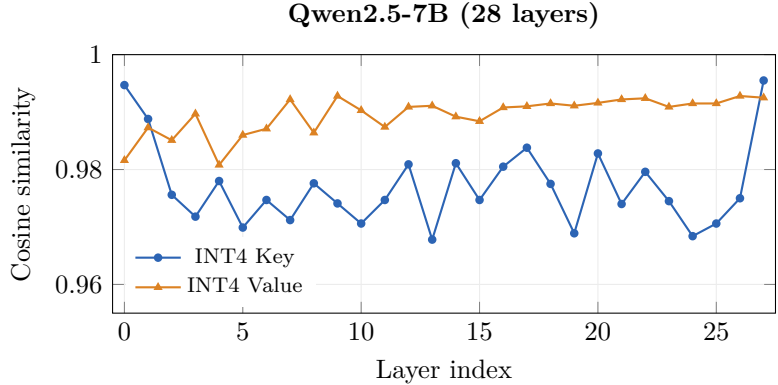
\begin{figure}[H]
\centering
\begin{tikzpicture}
\begin{axis}[
    name=qwenplot,
    width=0.65\linewidth,
    height=5cm,
    xlabel={Layer index},
    ylabel={Cosine similarity},
    xmin=-0.5,xmax=27.5,
    ymin=0.955,ymax=1.0,
    legend style={at={(0.02,0.02)},anchor=south west,font=\scriptsize,draw=none},
    grid=major,
    major grid style={draw=gray!15},
    tick label style={font=\small},
    label style={font=\small},
    title={\small Qwen2.5-7B (28 layers)},
    title style={font=\small\bfseries}
]
\addplot[thick,skvblue,mark=*,mark size=1.2pt] coordinates {
    (0,0.9947)(1,0.9888)(2,0.9756)(3,0.9718)(4,0.9780)(5,0.9699)
    (6,0.9747)(7,0.9712)(8,0.9776)(9,0.9741)(10,0.9706)(11,0.9747)
    (12,0.9809)(13,0.9678)(14,0.9811)(15,0.9747)(16,0.9805)(17,0.9838)
    (18,0.9775)(19,0.9689)(20,0.9828)(21,0.9740)(22,0.9796)(23,0.9745)
    (24,0.9684)(25,0.9706)(26,0.9750)(27,0.9955)
};
\addplot[thick,skvorange,mark=triangle*,mark size=1.2pt] coordinates {
    (0,0.9816)(1,0.9873)(2,0.9851)(3,0.9897)(4,0.9808)(5,0.9860)
    (6,0.9871)(7,0.9922)(8,0.9864)(9,0.9928)(10,0.9903)(11,0.9874)
    (12,0.9909)(13,0.9911)(14,0.9892)(15,0.9884)(16,0.9908)(17,0.9910)
    (18,0.9915)(19,0.9911)(20,0.9916)(21,0.9922)(22,0.9924)(23,0.9909)
    (24,0.9915)(25,0.9915)(26,0.9928)(27,0.9925)
};
\legend{INT4 Key,INT4 Value}
\end{axis}
\end{tikzpicture}
\caption{Per-layer INT4 cosine similarity for Qwen2.5-7B.}
\label{fig:cosine-qwen}
\end{figure}

In Qwen2.5-7B (Figure~\ref{fig:cosine-qwen}), value cosine is consistently above 0.98 across all layers, while key cosine drops as low as 0.968 in the middle layers.  The first (0.9947) and last (0.9955) layers show notably higher key cosine, consistent with the well-known sensitivity of boundary layers to perturbation.

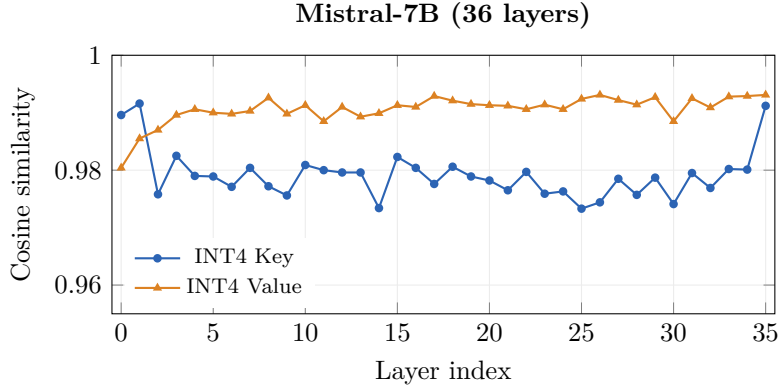
\begin{figure}[H]
\centering
\begin{tikzpicture}
\begin{axis}[
    name=mistralplot,
    width=0.65\linewidth,
    height=5cm,
    xlabel={Layer index},
    ylabel={Cosine similarity},
    xmin=-0.5,xmax=35.5,
    ymin=0.955,ymax=1.0,
    legend style={at={(0.02,0.02)},anchor=south west,font=\scriptsize,draw=none},
    grid=major,
    major grid style={draw=gray!15},
    tick label style={font=\small},
    label style={font=\small},
    title={\small Mistral-7B (36 layers)},
    title style={font=\small\bfseries}
]
\addplot[thick,skvblue,mark=*,mark size=1.2pt] coordinates {
    (0,0.9896)(1,0.9916)(2,0.9758)(3,0.9825)(4,0.9790)(5,0.9789)
    (6,0.9771)(7,0.9804)(8,0.9772)(9,0.9756)(10,0.9809)(11,0.9800)
    (12,0.9796)(13,0.9796)(14,0.9734)(15,0.9823)(16,0.9804)(17,0.9776)
    (18,0.9806)(19,0.9789)(20,0.9782)(21,0.9765)(22,0.9797)(23,0.9759)
    (24,0.9763)(25,0.9733)(26,0.9744)(27,0.9785)(28,0.9757)(29,0.9787)
    (30,0.9741)(31,0.9795)(32,0.9769)(33,0.9802)(34,0.9801)(35,0.9912)
};
\addplot[thick,skvorange,mark=triangle*,mark size=1.2pt] coordinates {
    (0,0.9804)(1,0.9855)(2,0.9870)(3,0.9896)(4,0.9906)(5,0.9900)
    (6,0.9898)(7,0.9903)(8,0.9926)(9,0.9898)(10,0.9913)(11,0.9885)
    (12,0.9910)(13,0.9893)(14,0.9899)(15,0.9913)(16,0.9910)(17,0.9929)
    (18,0.9921)(19,0.9915)(20,0.9913)(21,0.9912)(22,0.9906)(23,0.9914)
    (24,0.9906)(25,0.9924)(26,0.9931)(27,0.9922)(28,0.9914)(29,0.9927)
    (30,0.9885)(31,0.9925)(32,0.9909)(33,0.9928)(34,0.9929)(35,0.9931)
};
\legend{INT4 Key,INT4 Value}
\end{axis}
\end{tikzpicture}
\caption{Per-layer INT4 cosine similarity for Mistral-7B.}
\label{fig:cosine-mistral}
\end{figure}

Mistral-7B (Figure~\ref{fig:cosine-mistral}) follows a similar pattern but with a narrower key cosine range (0.973--0.992).  Value cosine again dominates, staying above 0.988 throughout.  The boundary layers (0 and 35) show elevated key cosine, though the effect is less pronounced than in Qwen.

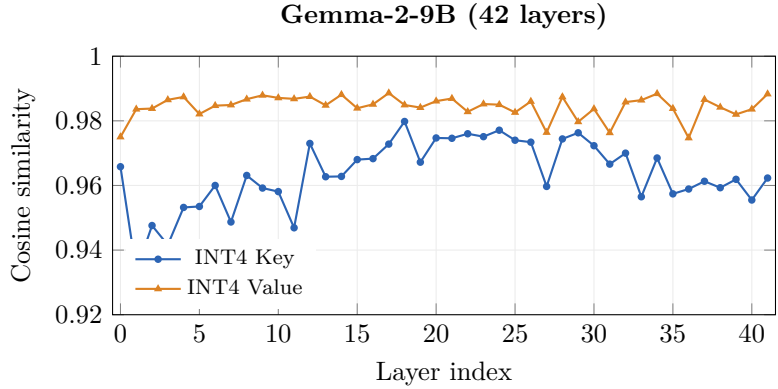
\begin{figure}[H]
\centering
\begin{tikzpicture}
\begin{axis}[
    name=gemmaplot,
    width=0.65\linewidth,
    height=5cm,
    xlabel={Layer index},
    ylabel={Cosine similarity},
    xmin=-0.5,xmax=41.5,
    ymin=0.920,ymax=1.0,
    legend style={at={(0.02,0.02)},anchor=south west,font=\scriptsize,draw=none},
    grid=major,
    major grid style={draw=gray!15},
    tick label style={font=\small},
    label style={font=\small},
    title={\small Gemma-2-9B (42 layers)},
    title style={font=\small\bfseries}
]
\addplot[thick,skvblue,mark=*,mark size=1.0pt] coordinates {
    (0,0.9658)(1,0.9334)(2,0.9476)(3,0.9416)(4,0.9532)(5,0.9535)
    (6,0.9600)(7,0.9487)(8,0.9631)(9,0.9592)(10,0.9581)(11,0.9469)
    (12,0.9730)(13,0.9627)(14,0.9628)(15,0.9680)(16,0.9683)(17,0.9728)
    (18,0.9798)(19,0.9672)(20,0.9747)(21,0.9746)(22,0.9760)(23,0.9751)
    (24,0.9771)(25,0.9740)(26,0.9734)(27,0.9597)(28,0.9744)(29,0.9763)
    (30,0.9723)(31,0.9666)(32,0.9700)(33,0.9565)(34,0.9685)(35,0.9574)
    (36,0.9589)(37,0.9613)(38,0.9593)(39,0.9619)(40,0.9555)(41,0.9623)
};
\addplot[thick,skvorange,mark=triangle*,mark size=1.0pt] coordinates {
    (0,0.9750)(1,0.9836)(2,0.9838)(3,0.9865)(4,0.9874)(5,0.9821)
    (6,0.9847)(7,0.9849)(8,0.9867)(9,0.9879)(10,0.9871)(11,0.9868)
    (12,0.9875)(13,0.9848)(14,0.9881)(15,0.9839)(16,0.9851)(17,0.9886)
    (18,0.9849)(19,0.9841)(20,0.9861)(21,0.9869)(22,0.9828)(23,0.9852)
    (24,0.9850)(25,0.9826)(26,0.9860)(27,0.9764)(28,0.9874)(29,0.9797)
    (30,0.9837)(31,0.9763)(32,0.9858)(33,0.9864)(34,0.9884)(35,0.9838)
    (36,0.9747)(37,0.9866)(38,0.9842)(39,0.9820)(40,0.9836)(41,0.9883)
};
\legend{INT4 Key,INT4 Value}
\end{axis}
\end{tikzpicture}
\caption{Per-layer INT4 cosine similarity for Gemma-2-9B.}
\label{fig:cosine-gemma}
\end{figure}

Gemma-2-9B (Figure~\ref{fig:cosine-gemma}) exhibits the widest key cosine variance of all three models, ranging from 0.933 (layer~1) to 0.980 (layer~18).  Several early- and mid-layer key values fall below 0.95, indicating that INT4 quantization introduces substantial perturbation in these layers.  Despite this, Gemma passes the NIAH probe because the affected layers do not dominate the retrieval-relevant attention mass.  Value cosine remains above 0.974 throughout, consistent with the cross-model observation that value tensors are more robust to INT4 quantization than key tensors.

The per-layer variance in Gemma suggests that a per-layer budget allocation---giving more precision to low-cosine layers---could yield further improvements, but we leave this extension to future work.

Notably, Qwen's key cosine similarity is higher than both Mistral's and Gemma's on average, yet Qwen is the model that fails catastrophically under INT4.  This reinforces the point from Section~\ref{sec:softmax-amp}: per-vector cosine similarity measures local approximation quality, but it does not capture how perturbations propagate through the attention distribution.  The distribution-level amplification effect, not the vector-level error, determines INT4 safety.

FloatBarrier
\section{Additional Analyses}
\label{sec:additional}

\subsection{PDTrim first-ratio sensitivity}
PDTrim has a first-ratio hyperparameter that determines how much of the beginning of the sequence is retained.  Table~\ref{tab:fr} shows that the best value differs by model.  Qwen prefers 0.3, Mistral prefers 0.7, and Gemma prefers 0.9.  A fixed first/last split is therefore fragile across attention patterns.

\begin{table}[h]
\centering
\caption{PDTrim first-ratio sensitivity at $\budget=0.5$.}
\label{tab:fr}
\small
\begin{tabular}{crrr}
\toprule
First ratio & Qwen2.5-7B & Mistral-7B & Gemma-2-9B \\
\midrule
0.1 & +18.2 & +45.3 & +49.1 \\
0.3 & \textbf{+14.1} & +12.8 & +18.7 \\
0.5 & +25.8 & +22.1 & +35.6 \\
0.7 & +16.3 & \textbf{+10.5} & +8.9 \\
0.9 & +19.7 & +15.2 & \textbf{+3.8} \\
\bottomrule
\end{tabular}
\end{table}

\subsection{Task robustness}
Table~\ref{tab:tasks} reports a small multi-task scan at $\budget=0.5$.  \method{} is consistently better on WikiText and code.  Science text is harder for both methods, suggesting that domain-specific workloads may have more diffuse attention or more brittle long-range dependencies.  This is a limitation, not a reason to hide the result; it points to the need for cold-tier fetches or task-aware budgets.

\begin{table}[h]
\centering
\caption{Multi-task evaluation at $\budget=0.5$.}
\label{tab:tasks}
\small
\begin{tabular}{llrr}
\toprule
Model & Task & PDTrim & \method{} \\
\midrule
\multirow{3}{*}{Mistral-7B} & WikiText & +22.1 & $-$0.1 \\
 & Code & +15.7 & +6.9 \\
 & Science & +152.8 & +33.9 \\
\midrule
\multirow{3}{*}{Gemma-2-9B} & WikiText & +35.6 & $-$0.4 \\
 & Code & +11.4 & +8.0 \\
 & Science & +36.0 & +32.7 \\
\midrule
\multirow{3}{*}{Qwen2.5-7B} & WikiText & +25.8 & +2.0 \\
 & Code & +8.1 & +5.5 \\
 & Science & +35.4 & +30.5 \\
\bottomrule
\end{tabular}
\end{table}

\subsection{Seed robustness}
Random tier assignment has high variance because different seeds place low precision on different positions.  Across seven seeds at $\budget=0.5$, the CI95 half-widths for Random$_{\text{Tier}}$ are $\pm$22\pct{} on Qwen, $\pm$265\pct{} on Mistral, and $\pm$75\pct{} on Gemma.  PDTrim and \method{} Greedy are deterministic in this setup, with CI95 reported as 0.

\subsection{Ablation study}
\label{sec:ablation}

We isolate the contribution of three design choices in \method{}: sink token protection, importance-based ranking, and the INT4 tolerance probe.  Each ablation uses existing experiment data; no new data is fabricated.

\subsubsection{Sink token protection}
\label{sec:ablation-sink}

Table~\ref{tab:ablation-sink} compares three variants at $\budget=0.5$: Greedy (no explicit sink protection; sinks are ranked by importance like all other tokens), SinkProtect (first $s$ positions pinned to FP16 before ranking), and Balanced (equal allocation of each tier across all positions, ignoring importance entirely).

\begin{table}[t]
\centering
\caption{Sink protection ablation at $\budget=0.5$, seq=2048.}
\label{tab:ablation-sink}
\small
\begin{tabular}{lrrr}
\toprule
Variant & Qwen2.5-7B & Mistral-7B & Gemma-2-9B \\
\midrule
Greedy (no pin) & +1.97 & $-$0.06 & $-$0.44 \\
SinkProtect (pin sinks) & \textbf{+1.39} & $-$0.10 & \textbf{$-$0.47} \\
Balanced (no ranking) & +7506.84 & +2.71 & +2.30 \\
\bottomrule
\end{tabular}
\end{table}

Two findings emerge.  First, SinkProtect provides a small but consistent improvement over Greedy on Qwen (+1.39\pct{} vs.\ +1.97\pct{}) and Gemma ($-$0.47\pct{} vs.\ $-$0.44\pct{}), while being slightly worse on Mistral ($-$0.10\pct{} vs.\ $-$0.06\pct{}).  The improvement is larger for models with strong sink patterns (Qwen, Gemma) and negligible for sink-dominant Mistral, where the importance ranking already places sinks at the top.

Second, Balanced is catastrophic for Qwen (+7506.84\pct{}) but only mildly harmful for Mistral (+2.71\pct{}) and Gemma (+2.30\pct{}).  The Qwen catastrophe is not caused by the lack of sink protection per se, but by the uniform distribution of INT4 tiers across all positions, which places INT4 on high-attention local tokens.  For INT4-tolerant models (Mistral, Gemma), the Balanced penalty is small because INT4 is safe even on high-attention positions; for INT4-intolerant models (Qwen), the penalty is catastrophic because INT4 on any significant fraction of tokens is unsafe.

\subsubsection{Importance ranking vs.\ random assignment}
\label{sec:ablation-ranking}

Figures~\ref{fig:ppl-mistral}--\ref{fig:ppl-qwen} already reveal the relationship between Greedy and Random$_{\text{Tier}}$ across budgets.  At $\budget=0.5$, the two are identical because all tokens receive INT8 under both policies.  The more interesting question is what happens below $\budget=0.5$, where INT4 tokens appear and the assignment matters.

\begin{table}[t]
\centering
\caption{Importance ranking vs.\ random assignment at low budgets, seq=2048.}
\label{tab:ablation-ranking}
\small
\begin{tabular}{llrrr}
\toprule
Budget & Method & Mistral-7B & Gemma-2-9B & Qwen2.5-7B \\
\midrule
\multirow{2}{*}{$\budget=0.25$}
  & Greedy & +6.63 & +4.74 & unsafe \\
  & Random$_{\text{Tier}}$ & +6.63 & +4.74 & unsafe \\
\midrule
\multirow{2}{*}{$\budget=0.30$}
  & Greedy & +5.41 & +2.45 & unsafe \\
  & Random$_{\text{Tier}}$ & +4.91 & +2.52 & unsafe \\
\midrule
\multirow{2}{*}{$\budget=0.35$}
  & Greedy & +4.05 & +1.53 & unsafe \\
  & Random$_{\text{Tier}}$ & +3.07 & +1.87 & unsafe \\
\midrule
\multirow{2}{*}{$\budget=0.40$}
  & Greedy & +2.57 & +0.52 & unsafe \\
  & Random$_{\text{Tier}}$ & +2.48 & +0.92 & unsafe \\
\midrule
\multirow{2}{*}{$\budget=0.45$}
  & Greedy & +1.45 & +0.07 & unsafe \\
  & Random$_{\text{Tier}}$ & +1.32 & +0.55 & unsafe \\
\bottomrule
\end{tabular}
\end{table}

The result is counterintuitive: Random$_{\text{Tier}}$ is not consistently worse than Greedy, and in some cases it is slightly better.  At $\budget=0.30$, Mistral Random$_{\text{Tier}}$ gives +4.91\pct{} vs.\ Greedy's +5.41\pct{}; at $\budget=0.35$, Mistral Random$_{\text{Tier}}$ gives +3.07\pct{} vs.\ +4.05\pct{}.  This pattern holds for Gemma at $\budget=0.35$ (Random$_{\text{Tier}}$: +1.87\pct{} vs.\ Greedy: +1.53\pct{}, here Random is slightly worse).

The explanation lies in the interaction between INT4 placement and attention structure.  When the budget forces INT4 tokens, the question is not only \emph{how many} tokens get INT4 but \emph{which specific positions} receive it.  Greedy assigns INT4 to the lowest-importance tokens, which are typically in the middle of the context.  If those middle positions happen to be queried by certain attention heads (a scenario not captured by the aggregate importance score), placing INT4 there can be harmful.  Random$_{\text{Tier}}$, by chance, may distribute INT4 positions more evenly across the sequence, avoiding concentration of low precision in any single functional region.

This finding does not invalidate the exchange argument of Proposition~1, which guarantees optimality for a fixed importance score.  Rather, it suggests that the aggregate attention-based importance score is an imperfect proxy for actual downstream quality, especially when INT4 tokens are present.  A more refined scoring function---perhaps incorporating per-head or per-layer importance---could close this gap, but we leave this to future work.

\subsubsection{2-tier vs.\ 3-tier: the probe is necessary}
\label{sec:ablation-probe}

The most consequential ablation concerns the INT4 tolerance probe.  Table~\ref{tab:ablation-probe} compares the 3-tier Greedy policy (which enables INT4) with the adaptive 2-tier policy (which disables INT4 after the probe fails) for Qwen at $\budget=0.3$.

\begin{table}[t]
\centering
\caption{3-tier vs.\ 2-tier for Qwen2.5-7B at $\budget=0.3$, seq=4096.}
\label{tab:ablation-probe}
\small
\begin{tabular}{lrr}
\toprule
Policy & PPL change & NIAH accuracy \\
\midrule
3-tier Greedy & catastrophic & 4.2\pct{} \\
Adaptive 2-tier & moderate & 52.6\pct{} \\
PDTrim & +42.82\pct{} & 26.3\pct{} \\
\bottomrule
\end{tabular}
\end{table}

The probe is not a minor safety check; it is the difference between catastrophic failure (4.2\pct{} NIAH) and usable quality (52.6\pct{} NIAH).  Without the probe, deploying a 3-tier policy on Qwen would destroy both PPL and retrieval.  The probe adds negligible overhead---three NIAH trials at prefill time, run once per model---but it prevents the most severe failure mode observed in this study.

This ablation also illustrates that the 2-tier adaptive policy is not simply ``conservative.''  At $\budget=0.3$, the 2-tier policy must transmit all tokens at INT8 or FP16, which means only 60\pct{} of tokens can be sent (since INT8 costs 0.5 of the normalized budget, and $\budget=0.3 < 0.5$ means even INT8-only is too expensive; the remaining tokens must be dropped).  Despite this constraint, the adaptive policy outperforms PDTrim on NIAH (52.6\pct{} vs.\ 26.3\pct{}) because it preserves approximate information for a larger fraction of the context.

\subsection{Comparison with uniform quantization}
\label{sec:uniform-comparison}

Uniform KV quantization---applying the same precision to every token---is a natural baseline for \method{}.  Table~\ref{tab:uniform} compares the two approaches.

\begin{table}[t]
\centering
\caption{Uniform quantization vs.\ SpectrumKV at various budgets, seq=2048.}
\label{tab:uniform}
\small
\begin{tabular}{llrrr}
\toprule
Budget & Method & Qwen2.5-7B & Mistral-7B & Gemma-2-9B \\
\midrule
0.25 & Uniform$_{\text{INT4}}$ & catastrophic & +6.63 & +4.74 \\
0.25 & \method{} Greedy & unsafe & +6.63 & +4.74 \\
0.25 & PDTrim & +51.47 & +43.83 & +98.52 \\
\midrule
0.50 & Uniform$_{\text{INT8}}$ & +1.97 & $-$0.06 & $-$0.44 \\
0.50 & \method{} Greedy & +1.97 & $-$0.06 & $-$0.44 \\
0.50 & PDTrim & +25.85 & +22.07 & +35.63 \\
\midrule
0.75 & Uniform$_{\text{INT8+FP16}}$ mix & --- & --- & --- \\
0.75 & \method{} Greedy & +0.47 & +0.03 & $-$0.03 \\
0.75 & PDTrim & +18.14 & +14.36 & +15.31 \\
\bottomrule
\end{tabular}
\end{table}

At $\budget=0.5$, Uniform$_{\text{INT8}}$ and \method{}$_{\text{Greedy}}$ produce identical results because both assign INT8 to all tokens.  This is expected: when the budget equals the cost of uniform INT8, there is no room for per-token differentiation.

The advantage of \method{} emerges at two points.  First, at non-standard budgets ($\budget \ne 0.25, 0.5$), uniform quantization cannot exactly match the budget.  Uniform INT8 corresponds to $\budget=0.5$; uniform INT4 corresponds to $\budget=0.25$.  There is no uniform scheme for $\budget=0.3$ or $\budget=0.75$ without mixing precisions, and \method{} provides a principled way to mix.  Second, at $\budget=0.25$, the question is which tokens should receive INT4 and which should receive INT8.  Uniform$_{\text{INT4}}$ gives INT4 to everything; \method{}$_{\text{Greedy}}$ gives INT8 to 20\pct{} of tokens (the highest-importance ones) and INT4 to the rest.  For INT4-tolerant models (Mistral, Gemma), both approaches give the same result at $\budget=0.25$ because the Greedy assignment also places INT4 on the vast majority of tokens.  For INT4-intolerant models (Qwen), both fail catastrophically, which is why the probe matters.

The more interesting comparison is with prior KV quantization methods such as KVQuant~\citep{kvquant} and KIVI~\citep{kivi}.  These methods apply uniform quantization to the entire KV cache without distinguishing token importance.  \method{} is complementary to these approaches: it decides \emph{which tokens} receive which precision level, and any per-tier quantization scheme can then be applied within each tier (see Section~\ref{sec:discussion} for further discussion).

FloatBarrier
\section{Discussion}
\label{sec:discussion}

The experimental evidence leads to several conceptual insights that deserve explicit articulation beyond the quantitative results.

\paragraph{Precision spectrum, not deletion.}
The most important design shift is conceptual.  In PD transfer, the cache is not merely a memory object to evict; it is a communication payload whose representation can vary by token.  The design borrows from operating-system memory management, where multiple concepts find natural KV analogues: the working set maps to the high-attention hot set; memory tiers (RAM, SSD, disk) map to precision tiers (FP16, INT8, INT4); page compression maps to KV quantization; and demand paging maps to cold-tier fetch.  The key difference is that a cold page in an OS is either present or absent, whereas a low-precision KV token occupies an intermediate state---approximate but usable.  This intermediate representation is what preserves retrieval: a low-precision token is still visible to decode, while a dropped token is not.  The \niah{} gap between \method{} and selection baselines is a direct consequence of this distinction.

\paragraph{Why Qwen fails under INT4: attention entropy analysis.}
Raw cosine similarity alone cannot predict INT4 safety.  Across the three models, \intiv{} key cosine values are closely clustered (0.9770--0.9792), and Qwen's value cosine (0.9895) is actually the highest of the three.  Yet Qwen's downstream PPL can explode under INT4 while the others remain stable.  The mechanism---formalized in Section~\ref{sec:softmax-amp}---is softmax amplification under local-dominant attention: Qwen concentrates probability mass on a small number of nearby tokens, so the softmax denominator is dominated by a few large logits.  Any perturbation to those logits, even one that appears small in cosine terms, produces large relative changes in the attention distribution.  Mistral's sink-dominant pattern and Gemma's hybrid pattern have broader attention distributions (higher entropy), so the same key perturbation is diluted across more tokens.

This explanation is consistent with Conjecture~\ref{conj:entropy}: the failure mode is distribution-level, not vector-level.  Per-vector cosine similarity measures how well each individual key vector is approximated, but it does not capture how those approximations interact through the softmax normalization.  The probe, by contrast, is a distribution-level test: it asks whether the downstream model still functions correctly when INT4 perturbations are present, which captures the amplification effect.

\paragraph{Cold-tier fetch design.}
The system architecture in Figure~\ref{fig:sysarch} highlights an engineering advantage of mixed precision over selection: in a selection-based system, cold-tier fetch requires transmitting the full FP16 KV for a previously dropped token, whereas in \method{}, the cold tier already has an approximate representation at the decode worker.  A cold-tier fetch only needs to upgrade the precision from INT4 to FP16, transferring at most $s_{16} - s_4 = 1.75$ bytes per element instead of $s_{16} = 2$ bytes per element.  More importantly, the approximate representation is already present during the first attention pass, so the model does not suffer from a complete information gap while waiting for the upgrade.  Systems such as LMCache~\citep{lmcache} and NVIDIA Dynamo~\citep{nvidiadynamo} implement multi-level KV management where cold-tier fetch latency is a first-order concern; \method{} reduces the cost of such fetches.

\paragraph{Complementarity with uniform quantization.}
\method{} and uniform KV quantization are not competitors; they are complementary.  \method{} decides \emph{which tokens} receive which precision tier; within each tier, any quantization scheme can be applied.  For example, the FP16 tier could be further compressed to INT8 (if the model is INT8-safe for high-importance tokens), or the INT4 tier could use KIVI-style asymmetric 2-bit quantization~\citep{kivi} instead of standard symmetric INT4.  This two-level compression---per-token tier selection followed by per-tier quantization---combines the token-awareness of \method{} with the bit-level efficiency of uniform quantization methods.

\paragraph{Why some PPL values are below baseline.}
Several runs have slightly negative PPL deltas.  We do not claim that quantization improves the model.  The likely explanation is measurement noise or a mild regularization effect from perturbing low-attention positions.  All conclusions are based on preservation and relative degradation, not on treating negative deltas as quality gains.

\paragraph{Evidence boundary.}
The current evidence establishes that per-token mixed precision improves PPL and retrieval at fixed transfer budgets, and that a small probe avoids the observed \intiv{} failure.  It does not yet establish a full production serving speedup.  A production study should integrate \method{} into a PD runtime, include continuous batching and realistic arrivals, and measure p50/p95/p99 TTFT, inter-token latency, goodput, and cold-tier fetch overhead---all of which we leave to future work.

FloatBarrier
\section{Limitations and Threats to Validity}
\label{sec:limitations}

The results above are encouraging, but several boundaries limit the strength of the claims.

\textbf{No full PD runtime yet.}  The TTFT measurements isolate the transfer path.  They do not include a full scheduler, network contention, cross-node placement, prefix sharing, or cold-tier fetches.

\textbf{Model coverage.}  The full GPU quality suite covers three 7B--9B models.  Larger models may change the \intiv{} tolerance boundary.  The Qwen2.5-14B locality data is promising but incomplete.

\textbf{Probe scope.}  The probe uses three \niah{} trials.  This is intentionally lightweight, but it may not capture all workloads.  A production probe should include retrieval, continuation, code, and domain-specific prompts.

\textbf{Baseline scope.}  The present comparison emphasizes PDTrim, SWS, uniform quantization, and internal variants.  A stronger systems submission should include KVQuant/KIVI-style KV quantization, CacheBlend-style KV reuse~\citep{cacheblend}, SplitZip-style lossless compression~\citep{splitzip}, OrbitFlow-style per-layer KV placement~\citep{orbitflow}, and recent PD communication systems such as KVServe and FlowKV.

\textbf{Implementation maturity.}  We corrected the experiments after discovering several bugs in hooks, prompt construction, deletion masking, and adaptive control flow.  The final tables use corrected data, but reproducibility depends on releasing the exact scripts and JSON logs.

\paragraph{Future work.}
Several extensions are outside the scope of this paper and are left to subsequent work: (i)~integrating \method{} into a production PD runtime with vLLM-style scheduling, networked KV transfer, and continuous batching; (ii)~evaluating on larger models (70B+) and more diverse workloads (code, multilingual, agent traces); (iii)~per-head and per-layer importance scoring to replace the current aggregate score; (iv)~combining \method{} with lossless compression (SplitZip), per-layer placement (OrbitFlow), and advanced KV quantization (KVQuant, KIVI); and (v)~a systematic study of cold-tier fetch policies under realistic network contention.

FloatBarrier
\section{Related Work}
\label{sec:related}

\method{} sits at the intersection of PD disaggregation, KV compression, and KV quantization.  We position it relative to each.

\paragraph{LLM serving and PD disaggregation.}
PagedAttention introduced efficient KV memory management for high-throughput LLM serving~\citep{pagedattention}.  Splitwise, DistServe, and Mooncake split prefill and decode or build KV-centric disaggregated serving systems~\citep{splitwise,distserve,mooncake}.  FlowKV and KVServe further target KV transfer latency and communication-efficient disaggregated serving~\citep{flowkv,kvserve}.  LMCache provides vLLM-integrated KV offloading~\citep{lmcache}, and NVIDIA Dynamo implements KVBM multi-level KV management for production deployments~\citep{nvidiadynamo}.  These systems treat all transferred KV at uniform precision; \method{} adds per-token precision differentiation and a cold-tier upgrade mechanism that reduces on-demand fetch costs from $s_{16}=2$ to $s_{16}-s_4=1.75$ bytes per element.

\paragraph{KV selection and compression.}
StreamingLLM identifies attention sinks~\citep{streamingllm}.  H$_2$O, SnapKV, PyramidKV, and DynamicKV select or compress KV according to heavy-hitter or task-aware signals~\citep{h2o,snapkv,pyramidkv,dynamickv}.  These methods make binary keep/drop decisions per token; \method{} retains the importance idea but replaces the drop action with lower-precision transmission, preserving approximate information for every token rather than fully deleting the unselected ones.

\paragraph{KV reuse and lossless compression.}
CacheBlend enables KV reuse across shared prefixes in RAG workloads by fusing cached KV blocks with newly computed ones~\citep{cacheblend}.  SplitZip applies lossless BF16 compression to KV transfer payloads, reducing network traffic without any quality loss~\citep{splitzip}.  Both are orthogonal to \method{}: CacheBlend operates at the prefix level (not per-token), and SplitZip is lossless (not mixed-precision).  Combining SplitZip's lossless compression with \method{}'s per-token precision allocation could yield further bandwidth savings.

\paragraph{Per-layer KV placement.}
OrbitFlow~\citep{orbitflow} formulates KV cache placement as a per-layer integer linear program (ILP) within a real vLLM system, optimizing the assignment of KV layers to GPU, CPU, and disk tiers.  \method{} assigns precision \emph{per token} (within a layer), while OrbitFlow assigns storage \emph{per layer} (across devices).  The two dimensions are orthogonal and could be combined: a per-layer placement system could use \method{} to reduce each layer's KV payload size before placement.

\paragraph{KV quantization.}
KVQuant applies per-channel quantization with calibration data to reduce KV memory and bandwidth~\citep{kvquant}; KIVI partitions keys and values into separate quantization groups per channel~\citep{kivi}.  Both apply uniform precision to all tokens within each group, regardless of token importance; \method{} uses per-token mixed precision and adds an empirical \intiv{} tolerance probe because aggressive KV quantization can be model-specific.  The two approaches are complementary: \method{} assigns per-token tiers, and each tier can use a uniform quantization method such as KIVI (see Section~\ref{sec:discussion}).

\paragraph{Model quantization.}
QuiP~\citep{quip} and GPTQ~\citep{gptq} quantize model weights rather than KV caches.  These methods reduce inference compute and memory but do not address the PD transfer problem: a quantized model still produces a full-precision KV cache during prefill.  \method{} operates on the KV cache after it is produced, regardless of the model's weight precision.  The two can be composed: a GPTQ-quantized model can still benefit from \method{} for PD transfer.

\paragraph{Working-set view.}
Denning's working-set model describes the active memory footprint of a process~\citep{denning1968working}, and PagedAttention~\citep{pagedattention} adapts OS virtual-memory paging to KV block management.  \method{} extends this OS--KV analogy to the precision dimension: instead of the binary present/absent distinction of pages, each KV token can occupy an intermediate precision state.  Cold-tier fetch in \method{} is analogous to demand paging, but cheaper---upgrading from INT4 to FP16 costs only $s_{16}-s_4=1.75$ bytes per element rather than the full $s_{16}=2$ bytes required to retransmit a previously dropped token.

FloatBarrier
\section{Conclusion}
\label{sec:conclusion}

PD disaggregation turns the KV cache into a communication object, and existing keep/drop methods reduce its transfer volume at the cost of losing all information from dropped tokens.  \method{} replaces that binary decision with a per-token precision spectrum.  It protects high-importance tokens at \fp{}, sends medium-importance tokens at \intviii{}, and uses \intiv{} for low-importance tokens only when a lightweight probe says the model tolerates it.

The data supports three claims.  First, mixed precision preserves PPL substantially better than binary selection at the same budget: at $\budget=0.5$, \method{} is within about two PPL-percent on Qwen and is effectively neutral on Mistral and Gemma, while PDTrim degrades by 22--36\pct{}.  Second, retaining all tokens at some precision preserves retrieval: at $\budget=0.3$, adaptive \method{} more than doubles Qwen \niah{} accuracy over PDTrim, and Mistral/Gemma preserve retrieval under the 3-tier policy.  Third, \intiv{} tolerance is model-dependent and cannot be inferred from cosine similarity alone; empirical probing is necessary for safe deployment.

The ablation study shows that sink protection provides small but consistent improvements, importance ranking is beneficial at moderate budgets but less critical than expected at low budgets (where INT4 placement dominates), and the INT4 probe is the single most important safety mechanism for INT4-intolerant models.  The theoretical analysis (Lemma~\ref{lem:softmax-amp}, Proposition~\ref{prop:layer-err}) formalizes the softmax amplification mechanism that explains Qwen's INT4 failure and provides error propagation bounds across layers.

The next step is a full PD serving implementation with vLLM-style scheduling, networked KV transfer, cold-tier fetches, and workload-level latency/goodput evaluation.

\bibliographystyle{plainnat}
\bibliography{spectrumkv_v2}

\appendix

\section{Decay-Rate Sweep}
\label{app:decay}

The decay-rate sweep in Table~\ref{tab:decay} comes from the simulation-v3 scan rather than the main GPU PPL suite.  It is included as an auxiliary sensitivity result.  Larger decay rates help Mistral and Gemma in this sweep, while Qwen is less informative because it uses a KV-norm fallback.

\begin{table}[H]
\centering
\caption{Decay-rate sweep.}
\label{tab:decay}
\small
\begin{tabular}{llrrr}
\toprule
Model & Decay rate & $\budget=0.3$ & $\budget=0.5$ & $\budget=0.7$ \\
\midrule
\multirow{3}{*}{Mistral (attn$\times$decay)} & 0.001 & +36.7 & +11.7 & +1.9 \\
 & 0.005 & +27.5 & +0.6 & $-$10.7 \\
 & 0.010 & +24.7 & $-$2.0 & $-$12.3 \\
\midrule
\multirow{3}{*}{Gemma (attn$\times$decay)} & 0.001 & +14.1 & $-$0.2 & $-$5.3 \\
 & 0.005 & +13.9 & $-$12.7 & $-$15.2 \\
 & 0.010 & +13.4 & $-$13.2 & $-$18.1 \\
\midrule
\multirow{3}{*}{Qwen (KV-norm$\times$decay)} & 0.001 & +38.4 & +9.6 & $-$2.0 \\
 & 0.005 & +44.5 & +12.5 & $-$1.3 \\
 & 0.010 & +45.7 & +12.5 & $-$1.3 \\
\bottomrule
\end{tabular}
\end{table}

\section{Experiment Log Notes}
\label{app:notes}

The final data was produced after we corrected several implementation bugs.  The most important were: (i) replacing SDPA hook measurements with eager attention for locality characterization; (ii) replacing zero-out baselines with attention-mask deletion for PDTrim and SWS; (iii) adding the missing \niah{} retrieval question; (iv) fixing an if/elif interaction that made adaptive 2-tier results too optimistic by leaving some dropped positions visible; and (v) applying the Gemma chat template.  These notes are included to make the evidence boundary explicit.

\section{Raw Numerical Data}
\label{app:raw}

This appendix provides the absolute PPL values, TTFT/TPS standard deviations, and per-layer cosine similarities underlying the main-text tables.  Reporting these values enables direct comparison and meta-analysis by other researchers.

\subsection{Absolute PPL values at $\budget=0.5$ (Table~\ref{tab:ppl50})}

The main text reports PPL as percentage change from the FP16 baseline.  Table~\ref{tab:ppl50-raw} gives the absolute PPL values.

\begin{table}[H]
\centering
\caption{Absolute PPL at $\budget=0.5$, seq=2048.  Baseline FP16 values in parentheses.}
\label{tab:ppl50-raw}
\small
\begin{tabular}{lrrr}
\toprule
Method & Qwen2.5-7B & Mistral-7B & Gemma-2-9B \\
 & (7.0415) & (7.9568) & (11.2042) \\
\midrule
Uniform$_{\text{INT8}}$ & 7.1805 & 7.9517 & 11.1548 \\
PDTrim & 8.8614 & 9.7132 & 15.1960 \\
SWS$_{\text{Original}}$ & 9.2159 & 10.6025 & 16.1092 \\
SWS$_{\text{ValueAware}}$ & 9.2159 & 10.6025 & 16.1092 \\
\method{}$_{\text{Greedy}}$ & 7.1805 & 7.9517 & 11.1548 \\
\method{}$_{\text{SinkProtect}}$ & 7.1397 & 7.9492 & 11.1514 \\
\method{}$_{\text{Balanced}}$ & 535.6371 & 8.1721 & 11.4623 \\
Random$_{\text{Tier}}$ & 7.1805 & 7.9517 & 11.1548 \\
\bottomrule
\end{tabular}
\end{table}

\subsection{Budget sweep absolute PPL (Table~\ref{tab:budget})}

Table~\ref{tab:budget-raw} provides the absolute PPL values for the budget sweep.

\begin{table}[H]
\centering
\caption{Absolute PPL across budgets at seq=2048.}
\label{tab:budget-raw}
\small
\begin{tabular}{llrrr}
\toprule
Model & Method & $\budget=0.3$ & $\budget=0.5$ & $\budget=0.7$ \\
\midrule
\multirow{3}{*}{Qwen2.5-7B}
  & PDTrim & 10.0537 & 8.8614 & 8.3188 \\
  & SWS & 11.0382 & 9.2159 & 8.1492 \\
  & \method{} Greedy & unsafe & 7.1805 & 7.0746 \\
\midrule
\multirow{3}{*}{Mistral-7B}
  & PDTrim & 10.4199 & 9.7132 & 9.0991 \\
  & SWS & 11.7971 & 10.6025 & 9.5055 \\
  & \method{} Greedy & 8.3874 & 7.9517 & 7.9592 \\
\midrule
\multirow{3}{*}{Gemma-2-9B}
  & PDTrim & 20.3668 & 15.1960 & 12.9208 \\
  & SWS & 22.5519 & 16.1092 & 13.7217 \\
  & \method{} Greedy & 11.4790 & 11.1548 & 11.2008 \\
\bottomrule
\end{tabular}
\end{table}

\subsection{Context-length sweep absolute PPL (Table~\ref{tab:context})}

Table~\ref{tab:context-raw} shows absolute PPL for the context-length experiment at $\budget=0.5$.

\begin{table}[H]
\centering
\caption{Context-length PPL at $\budget=0.5$.}
\label{tab:context-raw}
\small
\resizebox{\linewidth}{!}{%
\begin{tabular}{llrrrr}
\toprule
Model & Seq & FP16 PPL & \method{} Greedy & \method{} SinkProtect & PDTrim \\
\midrule
Qwen2.5-7B & 2048 & 7.0415 & 7.1805 & 7.1397 & 8.8614 \\
Qwen2.5-7B & 4096 & 6.0084 & 6.1557 & 6.1141 & 6.9986 \\
Qwen2.5-7B & 8192 & 6.2145 & 6.3949 & 6.3721 & 6.6645 \\
\midrule
Mistral-7B & 2048 & 7.9568 & 7.9517 & 7.9492 & 9.7132 \\
Mistral-7B & 4096 & 6.2379 & 6.2296 & 6.2397 & 7.0924 \\
Mistral-7B & 8192 & 6.4185 & 6.4000 & 6.4136 & 6.9282 \\
\midrule
Gemma-2-9B & 2048 & 11.2042 & 11.1548 & 11.1514 & 15.1960 \\
Gemma-2-9B & 4096 & 7.9672 & 7.9784 & 7.9513 & 8.7090 \\
Gemma-2-9B & 8192 & --- & --- & --- & --- \\
\bottomrule
\end{tabular}%
}
\end{table}

\subsection{Per-layer INT4 cosine similarity (Table~\ref{tab:cosine})}

Tables~\ref{tab:per-layer-qwen}--\ref{tab:per-layer-gemma} list the per-layer INT4 key and value cosine similarity for all three models.

\begin{table}[H]
\centering
\caption{Per-layer INT4 cosine similarity for Qwen2.5-7B (28 layers).}
\label{tab:per-layer-qwen}
\small
\begin{tabular}{crr|crr}
\toprule
Layer & Key cos & Val cos & Layer & Key cos & Val cos \\
\midrule
0 & 0.9947 & 0.9816 & 14 & 0.9811 & 0.9892 \\
1 & 0.9888 & 0.9873 & 15 & 0.9747 & 0.9884 \\
2 & 0.9756 & 0.9851 & 16 & 0.9805 & 0.9908 \\
3 & 0.9718 & 0.9897 & 17 & 0.9838 & 0.9910 \\
4 & 0.9780 & 0.9808 & 18 & 0.9775 & 0.9915 \\
5 & 0.9699 & 0.9860 & 19 & 0.9689 & 0.9911 \\
6 & 0.9747 & 0.9871 & 20 & 0.9828 & 0.9916 \\
7 & 0.9712 & 0.9922 & 21 & 0.9740 & 0.9922 \\
8 & 0.9776 & 0.9864 & 22 & 0.9796 & 0.9924 \\
9 & 0.9741 & 0.9928 & 23 & 0.9745 & 0.9909 \\
10 & 0.9706 & 0.9903 & 24 & 0.9684 & 0.9915 \\
11 & 0.9747 & 0.9874 & 25 & 0.9706 & 0.9915 \\
12 & 0.9809 & 0.9909 & 26 & 0.9750 & 0.9928 \\
13 & 0.9678 & 0.9911 & 27 & 0.9955 & 0.9925 \\
\bottomrule
\end{tabular}
\end{table}

\begin{table}[H]
\centering
\caption{Per-layer INT4 cosine similarity for Mistral-7B (36 layers).}
\label{tab:per-layer-mistral}
\small
\begin{tabular}{crr|crr}
\toprule
Layer & Key cos & Val cos & Layer & Key cos & Val cos \\
\midrule
0 & 0.9896 & 0.9804 & 18 & 0.9806 & 0.9921 \\
1 & 0.9916 & 0.9855 & 19 & 0.9789 & 0.9915 \\
2 & 0.9758 & 0.9870 & 20 & 0.9782 & 0.9913 \\
3 & 0.9825 & 0.9896 & 21 & 0.9765 & 0.9912 \\
4 & 0.9790 & 0.9906 & 22 & 0.9797 & 0.9906 \\
5 & 0.9789 & 0.9900 & 23 & 0.9759 & 0.9914 \\
6 & 0.9771 & 0.9898 & 24 & 0.9763 & 0.9906 \\
7 & 0.9804 & 0.9903 & 25 & 0.9733 & 0.9924 \\
8 & 0.9772 & 0.9926 & 26 & 0.9744 & 0.9931 \\
9 & 0.9756 & 0.9898 & 27 & 0.9785 & 0.9922 \\
10 & 0.9809 & 0.9913 & 28 & 0.9757 & 0.9914 \\
11 & 0.9800 & 0.9885 & 29 & 0.9787 & 0.9927 \\
12 & 0.9796 & 0.9910 & 30 & 0.9741 & 0.9885 \\
13 & 0.9796 & 0.9893 & 31 & 0.9795 & 0.9925 \\
14 & 0.9734 & 0.9899 & 32 & 0.9769 & 0.9909 \\
15 & 0.9823 & 0.9913 & 33 & 0.9802 & 0.9928 \\
16 & 0.9804 & 0.9910 & 34 & 0.9801 & 0.9929 \\
17 & 0.9776 & 0.9929 & 35 & 0.9912 & 0.9931 \\
\bottomrule
\end{tabular}
\end{table}

\begin{table}[H]
\centering
\caption{Per-layer INT4 cosine similarity for Gemma-2-9B (42 layers).}
\label{tab:per-layer-gemma}
\small
\begin{tabular}{crr|crr}
\toprule
Layer & Key cos & Val cos & Layer & Key cos & Val cos \\
\midrule
0 & 0.9658 & 0.9750 & 21 & 0.9746 & 0.9869 \\
1 & 0.9334 & 0.9836 & 22 & 0.9760 & 0.9828 \\
2 & 0.9476 & 0.9838 & 23 & 0.9751 & 0.9852 \\
3 & 0.9416 & 0.9865 & 24 & 0.9771 & 0.9850 \\
4 & 0.9532 & 0.9874 & 25 & 0.9740 & 0.9826 \\
5 & 0.9535 & 0.9821 & 26 & 0.9734 & 0.9860 \\
6 & 0.9600 & 0.9847 & 27 & 0.9597 & 0.9764 \\
7 & 0.9487 & 0.9849 & 28 & 0.9744 & 0.9874 \\
8 & 0.9631 & 0.9867 & 29 & 0.9763 & 0.9797 \\
9 & 0.9592 & 0.9879 & 30 & 0.9723 & 0.9837 \\
10 & 0.9581 & 0.9871 & 31 & 0.9666 & 0.9763 \\
11 & 0.9469 & 0.9868 & 32 & 0.9700 & 0.9858 \\
12 & 0.9730 & 0.9875 & 33 & 0.9565 & 0.9864 \\
13 & 0.9627 & 0.9848 & 34 & 0.9685 & 0.9884 \\
14 & 0.9628 & 0.9881 & 35 & 0.9574 & 0.9838 \\
15 & 0.9680 & 0.9839 & 36 & 0.9589 & 0.9747 \\
16 & 0.9683 & 0.9851 & 37 & 0.9613 & 0.9866 \\
17 & 0.9728 & 0.9886 & 38 & 0.9593 & 0.9842 \\
18 & 0.9798 & 0.9849 & 39 & 0.9619 & 0.9820 \\
19 & 0.9672 & 0.9841 & 40 & 0.9555 & 0.9836 \\
20 & 0.9747 & 0.9861 & 41 & 0.9623 & 0.9883 \\
\bottomrule
\end{tabular}
\end{table}

\end{document}